%% file: document.tex
\documentclass{article}
\usepackage{url}            % simple URL typesetting
\usepackage{booktabs}       % professional-quality tables
\usepackage{amsfonts}       % blackboard math symbols
\usepackage{nicefrac}       % compact symbols for 1/2, etc.
\usepackage{microtype}      % microtypography

\usepackage{graphicx}
\usepackage{subcaption}
\usepackage{amsmath}
\usepackage{amsthm}
\usepackage{amssymb}
\usepackage{mathtools}
\usepackage{natbib}
\setcitestyle{numbers,square}
\usepackage{thmtools}
\usepackage{thm-restate}
\usepackage{diagbox}
\usepackage{multirow}
\usepackage[final]{neurips_2024}
\usepackage{hyperref}
\newcommand{\bx}{\boldsymbol{x}}

\newcommand{\be}{\boldsymbol{e}}

\newcommand{\beps}{\boldsymbol{\epsilon}}

\newcommand{\by}{\boldsymbol{y}}

\newcommand{\bP}{\boldsymbol{P}}

\newcommand{\btheta}{\boldsymbol{\theta}}

\newcommand{\bI}{\boldsymbol{I}}

\newcommand{\bbP}{\mathbb{P}}

\newcommand{\diag}{\mathsf{diag}}

\newtheorem{remark}{\textbf{Remark}}

\newcommand{\mE}{\mathbb{E}}

\newcommand{\cC}{\mathcal{C}}
\newcommand{\cN}{\mathcal{N}}

\newcommand{\cO}{\mathcal{O}}

\newcommand{\ri}{\mathrm{i}}

\newcommand{\baralpha}{\bar{\alpha}}

\makeatother

\usepackage[textsize=tiny]{todonotes}

\title{Towards Understanding the Working Mechanism of Text-to-Image Diffusion Model}

\author{Mingyang Yi$^{1*}$, Aoxue Li$^{2*}$, Yi Xin$^{3*}$, Zhenguo Li$^{2}$\\
	$^{1}$ Renmin University of China\\ 
	$^{2}$ Huawei Noah's Ark Lab\\
	$^{3}$ Nanjing University\\
	\texttt{\{yimingyang@ruc.edu.cn\}} \\
	\texttt{\{liaoxue2,li.zhenguo\}@huawei.com} \\
	\texttt{xinyi@smail.nju.edu.cn}
}

\begin{document}
	\maketitle
	\def\thefootnote{*}\footnotetext{equal contribution}
	\def\thefootnote{\dag}\footnotetext{corresponding to Mingyang Yi: yimingyang@ruc.edu.cn}
	
	
	
	
	\begin{abstract}
		Recently, the strong latent Diffusion Probabilistic Model (DPM) has been applied to high-quality Text-to-Image (T2I) generation (e.g., Stable Diffusion), by injecting the encoded target text prompt into the gradually denoised diffusion image generator. Despite the success of DPM in practice, the mechanism behind it remains to be explored. To fill this blank, we begin by examining the intermediate statuses during the gradual denoising generation process in DPM. The empirical observations indicate, the shape of image is reconstructed after the first few denoising steps, and then the image is filled with details (e.g., texture). The phenomenon is because the low-frequency signal (shape relevant) of the noisy image is not corrupted until the final stage in the forward process (initial stage of generation) of adding noise in DPM. Inspired by the observations, we proceed to explore the influence of each token in the text prompt during the two stages. After a series of experiments of T2I generations conditioned on a set of text prompts. We conclude that in the earlier generation stage, the image is mostly decided by the special token [\texttt{EOS}] in the text prompt, and the information in the text prompt is already conveyed in this stage. After that, the diffusion model completes the details of generated images by information from themselves. Finally, we propose to apply this observation to accelerate the process of T2I generation by properly removing text guidance, which finally accelerates the sampling up to 25\%+. % \yi{20}.
		
		% Finally, we propose to apply this observation to accelerate T2I generation and latent consistency model training by properly removing text guidance, which finally accelerates the sampling and training respectively up to 25\%+ and% \yi{20}.
	\end{abstract}
	\section{Introduction}
	In real-world application, the Text-to-Image (T2I) generation has long been explored owing to its wide applications \citep{yu2022scaling,ramesh2022hierarchical,rombach2022high,bao2023all,kang2023scaling}, whereas the Diffusion Probabilistic Model (DPM) \citep{ho2020denoising,song2020score,rombach2022high} stands out as a promising approach, thanks to its impressive image synthesis capability. Technically, the DPM is a hierarchical denoising model, which gradually purifies noisy data from a standard Gaussian to generate an image. 
	In the existing literature \citep{ramesh2022hierarchical,saharia2022photorealistic,peebles2023scalable,croitoru2023diffusion}, the framework of (latent) Stable Diffusion model \citep{ramesh2022hierarchical} is a backbone technique in T2I generation via DPM. In this approach, the text prompt is encoded by a CLIP text encoder \citep{radford2021learning}, and injected into the diffusion image decoder as a condition to generate a target image (latent encoded by VQ-GAN in \citep{ramesh2022hierarchical}) that is consistent with the text prompt. Though this framework works well in practice, the working mechanism behind it, especially for the text prompt, remains to be explored. Therefore, in this paper, we systematically explore the working mechanism of stable diffusion.
	\par
	Our investigation starts from the intermediate status of the denoising generation process. Through an experiment (details are in Section \ref{sec:First Overall Shape then Details}), we find that in the early stage of the denoising process, the overall shapes of generated images (latent) are already reconstructed. In contrast, the details  (e.g., textures) are then filled at the end of the denoising process. To explain this, we notice that the overall shape (resp. semantic details) is decided by low-frequency (resp. high-frequency) signals \citep{gonzales1987digital}. We both empirically show and theoretically explain that in contrast to the high-frequency signals, the low-frequency signals of noisy data are not corrupted until the end stage of the forward noise-adding process. Therefore, its reverse denoising process firstly recovers the low-frequency signal (so that overall shape) in the initial stage, and then recovers the high-frequency part in the latter stage. 
	\par
	Following the phenomenons, we investigate the effect of encoded tokens in the text prompt of T2I generation during the two stages, where each token is encoded by an auto-regressive CLIP text encoder. The text prompt has a length of 76, and is enclosed by special tokens [\texttt{SOS}] and [\texttt{EOS}], \footnote{[\texttt{EOS}] contains the overall information in text prompt due to the auto-regressive CLIP text-encoder} at the beginning and end of the text prompt, respectively. Therefore, we categorize the tokens into three classes, i.e., [\texttt{SOS}], semantic tokens, and [\texttt{EOS}]. Notably, the special token [\texttt{SOS}] does not contain information, due to the auto-regressive encoding of the text prompt. Thus, our investigations into the influence of tokens will primarily focus on the semantic tokens and [\texttt{EOS}]. Surprisingly, we find that compared with semantic tokens, the special token [\texttt{EOS}] has a larger impact during generation. 
	\par
	Concretely, under a set of collected text prompts, we select 1000 pairs ``[\texttt{SOS}] + Prompt $A$ ($B$) + [\texttt{EOS}]$_{A(B)}$'' from it. Then, replace the special token [\texttt{EOS}]$_{A}$ in the text prompt $A$ with [\texttt{EOS}]$_{B}$ from prompt $B$ to observe the generated data under this condition. Interestingly, we find that the generated images are more likely to be aligned with text prompt $B$ (especially for the shape features) instead of $A$, so that [\texttt{EOS}] has a larger impact compared with semantic tokens. Besides that, we further find that the information in [\texttt{EOS}] is already conveyed during the early shape reconstruction stage of the denoising process. Exploring along the working stage of [\texttt{EOS}], we further verify and explain that the whole text prompts (including semantic ones) primarily work on the early denoising process, when the overall shapes of generated images are constructed. After that, the image details are mainly reconstructed by the images themselves. This phenomenon is explained by ``first shape then details'', as the injected text prompt implicitly penalizes the generated images to be consistent with it. Therefore, the penalization quickly becomes weak, when the overall shape of image is reconstructed.  
	
	% All of our empirical observations are quantitatively verified by three standard metrics CLIPScore \citep{radford2021learning,hessel2021clipscore}, BLIP-VQA \citep{li2022blip,huang2023t2i}, and MiniGPT4-CoT \citep{zhu2023minigpt,huang2023t2i}.
	\par 
	Finally, we apply our observations in one practical cases: Training-free sampling acceleration, as the text prompt works in the first stage of denoising process, we remove the textual prompt-related model propagation ($\beps_{\btheta}(t, \bx_{t}, \cC)$ in \eqref{eq:noise prediction}) during the details reconstruction stage, which merely change the generated images but save about 25\%+ inference cost. 
	% 2): Latent Consistency Model (LCM) \citep{luo2023latent} training acceleration, the LCM framework distillates the multi-step denoising process into one-step model forward propagation, where the distillation involves conducting denoising steps. As in the first application, we remove the textual prompt-related model propagation when it corresponds to the second denoising stage, during the distillation of LCM. By doing so, the trained LCM generates images comparable to the original ones but saves \yi{20}\% computational cost in training.     
	\par
	We summarize our contributions as follows. 
	\begin{enumerate}
		\item We show, during the denoising process of the stable diffusion model, the overall shape and details of generated images are respectively reconstructed in the early and final stages of it. 
		\item For the working mechanism of text prompt, we empirically show the special token [\texttt{EOS}] dominates the influence of text prompt in the early (overall shape reconstruction) stage of denoising process, when the information from text prompt is also conveyed. Subsequently, the model works on filling the details of generated images mainly depending on themselves. 
		\item We apply our observation to accelerate the sampling of denoising process 25\%+.    
	\end{enumerate}
	\section{Related Work}
	\paragraph{Diffusion Model.} In this paper, our exploration is based on the Stable Diffusion \citep{rombach2022high}, which now terms to be a standard T2I generation technique based on DPM \citep{ho2020denoising,song2020denoising}, and has been applied into various computer vision domains e.g., 3D  \citep{poole2022dreamfusion,ruiz2023dreambooth,lin2023magic3d} and video generation \citep{molad2023dreamix,chai2023stablevideo}. In practice, the goal of T2I is generating an image that is consistent with a given text injected into the cross-attention module \citep{vaswani2017attention} of the image decoder. Therefore, understanding the working mechanism of stable diffusion potentially improves the existing techniques \citep{balaji2022ediffi}. Unfortunately, to the best of our knowledge, the problem is limited explored, expected in \citep{yang2023diffusion,si2023freeu}, where they similarly observe the low-frequency signals are firstly recovered in the denoising process. However, further explanations for this phenomenon are neglected in these works. % do not further explain such phenomenon.  
	\paragraph{Influence of Tokens.} Understanding the working mechanism of encoded (by a pretrained language model) text prompt  \citep{brown2020language,radford2021learning,wei2022chain,openai2023gpt4} helps us understanding T2I generation \citep{sun2021long,krishna2022rankgen,liu2023lost}. For example, \citep{xiao2023efficient} finds that in LLM, the first token primarily decides the weights in the cross-attention map, which similarly appeared in the cross-modality text-image stable diffusion model as we observed. \citep{zhang2021counterfactual} explores the influence of individual tokens in counterfactual memorization. However, in the multi-modality models e.g., \citep{radford2021learning,li2022blip,li2023blip,luddecke2022image}, whereas the textual information interacted with the image in the cross-attention module as in stable diffusion, the working mechanism of tokens interacts with cross-modality data is limited explored, expected in \citep{balaji2022ediffi}. They find in a single case that the influence of text prompts may decrease during the denoising process, while they do not proceed to study or apply this phenomenon as in this paper. Recently, \citep{zhang2024cross} finds that the cross-attention map between the text prompt and generated images converges during the denoising process, which is also explained by our observations that the information conveyed during the first few denoising steps. Besides that, unlike ours, their observations are lack of theoretical explanation.  
	\section{Preliminaries}\label{sec:preliminaries}
	We briefly introduce the (latent) stable diffusion model \citep{ramesh2022hierarchical}, which transfers a standard Gaussian noise into a target image latent that aligns with pre-given text prompts. Here, the generated data space is a low-dimensional Vector-Quantized (VQ) \citep{esser2021taming} image latent to reduce the computational cost of generation. One may get the target natural image by decoding the generated image latent. In this paper, the original data (image latent) is denoted by $\bx_{0}$, and the encoded textual prompt (by CLIP text encoder \citep{radford2021learning}) is represented by $\cC$. The noisy data
	\begin{equation}\label{eq:noisy data}
		\small
		\bx_{t} = \sqrt{\baralpha_{t}}\bx_{0} + \sqrt{1 - \baralpha_{t}}\beps_{t},
	\end{equation}
	is used as input to diffusion model $\beps_{\btheta}$ trained by 
	\begin{equation}
		\small
		\min_{\btheta}\mE\left[\left\|\beps_{\btheta}(t, \bx_t, \cC, \emptyset) - \beps_{t}\right\|^{2}\right],
	\end{equation}
	with $0\leq t\leq T$, $\beps_{t}\sim\cN(0, \bI)$ independent of $\bx_{0}$, $\baralpha_{t}\to 0$ (resp. $\baralpha_{t}\to 1$) for $t\to 0$ (resp. $t\to T$). Here, the noise prediction model $\beps_{\btheta}(t, \bx_t, \cC, \emptyset)$ is constructed by classifier-free guidance \citep{ho2021classifier} with 
	\begin{equation}\label{eq:noise prediction}
		\small
		\beps_{\btheta}(t, \bx_t, \cC, \emptyset) = \beps_{\btheta}(t, \bx_t, \emptyset) + w\left(\beps_{\btheta}(t, \bx_t, \cC) - \beps_{\btheta}(t, \bx_t, \emptyset)\right),
	\end{equation}
	where $\beps_{\btheta}(t, \bx_t, \emptyset)$ is an unconditional generative model, and the $w \geq 0$ is guidance scale. As the model is trained to predict noise $\beps_{t}$ in $\bx_{t}$, and $\bx_{T}$ approximates a standard Gaussian, we can conduct the reverse denoising process (DDIM \citep{song2020denoising}) transfers a standard Gaussian to target image $\bx_{0}$
	\begin{equation}\label{eq:forward propogation}
		\small
		\bx_{t - 1} = \sqrt{\frac{\baralpha_{t - 1}}{\baralpha_{t}}}\bx_{t} + \left(\sqrt{\frac{1 - \baralpha_{t - 1}}{\baralpha_{t - 1}}} - \sqrt{\frac{1 - \baralpha_{t}}{\baralpha_{t}}}\right)\beps_{\btheta}(t, \bx_{t}, \cC, \emptyset). 
	\end{equation}
	\par
	Finally, the diffusion model (usually UNet) takes the text prompt as input to the cross-attention module in each basic block \footnote{The model is stacked by such basic blocks sequentially with residual module, self-attention module, and cross-attention module in each of it.} of diffusion model with output ${\rm Attention}(Q, K, V) = \mathrm{Softmax}(QK^{\top} / \sqrt{d})V$ ($d$ is dimension of image feature), where $\phi(\bx_{t})$ is the feature of image, and 
	\begin{equation}\label{eq:cross attention}
		\small
		Q = W_{Q}\phi(\bx_{t}); K = W_{K}\cC; V = W_{V}\cC.
	\end{equation} 
	\section{First Overall Shape then Details}\label{sec:First Overall Shape then Details}
	In this section, we first explore the image reconstruction process of the stable diffusion model. As noted in \citep{balaji2022ediffi}, the generated image's overall shape is difficult to be alterted in the final stage of the denoising process. Inspired by this observation, and note that the low-frequency and high-frequency signals of image determine its overall shape and details, respectively \citep{gonzales1987digital}. We theoretically and empirically verify that the denoising process recovers the low and high-frequency signals in its initial and final stages, respectively, which explains the phenomenon of  \emph{``\textbf{first overall shape then details''}}.  
	\subsection{Two Stages of Denoising Process}\label{sec:Two Stages of Denoising Process}
	\begin{figure*}[t!]
		\centering
		\vspace{-0.6in}
		\subfloat[Cross-Attention Map\label{fig:cross-attention map}]{\includegraphics[scale=0.35]{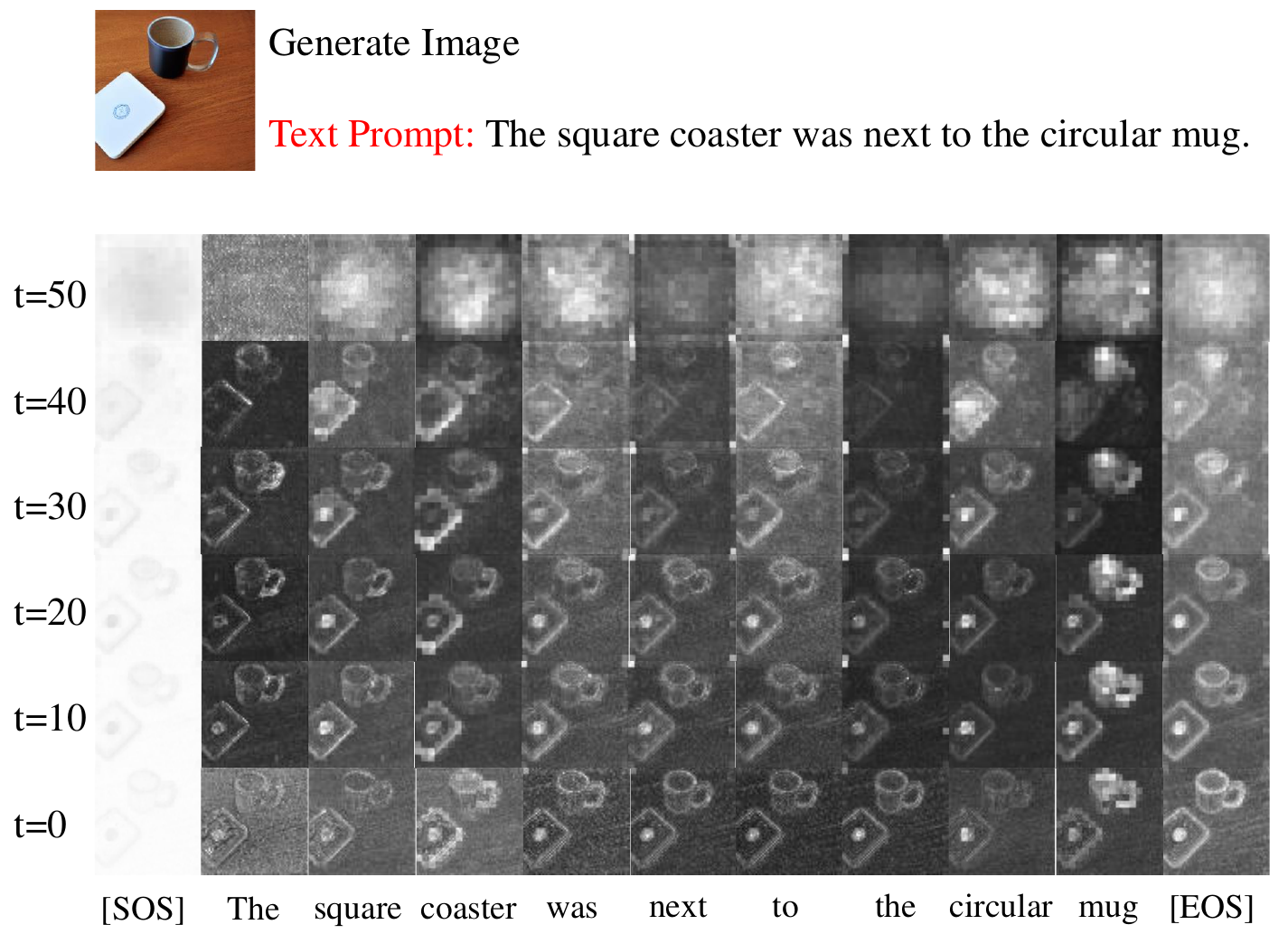}}
		\subfloat[Convergence of Cross-Attention Map\label{fig:cross-attention map ratio}]{\includegraphics[scale=0.3]{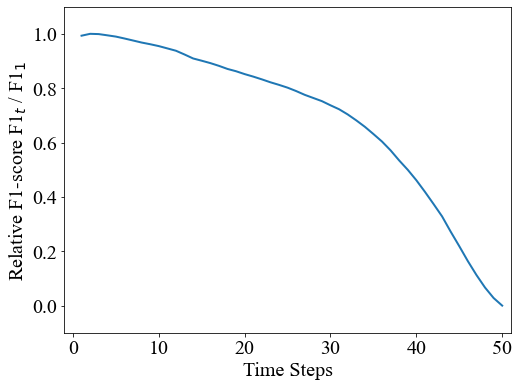}}
		
		\caption{Figure \ref{fig:cross-attention map} is the averaged cross-attention over denoising steps. The two generated images are on the top, and the weights in cross-attention maps of each tokens are on the bottom with whiter pixels correspond to larger weights in cross-attention map. Figure \ref{fig:cross-attention map ratio} is obtained by taking average over tokens and prompts in \texttt{PromptSet}, which compares the shapes of cross-attention map and final generated images, Measured by relative F1-score $\text{F1}_{t} / \text{F1}_{1}$ over different denoising steps.}
		\vspace{-0.2in}
	\end{figure*}
	%	We will start our analysis from the intermediate status during the whole denoising process. 
	\paragraph{Settings (\texttt{PromptSet}).} As in \citep{huang2023t2i}, we use 1600 prompts following the template ``a \{attribute\} \{noun\}'', with the attribute as an adjective of color or texture. We create 800 text prompts respectively under each of the two categories of attributes. Besides that, we add another extra 1000 complex natural prompts in \citep{huang2023t2i} without a predefined sentence template. These prompts consist of the text prompts set (abbrev \texttt{PromptSet}) we used. The classes of nouns, colors, and textures are respectively 230, 33, and 23 in these prompts. In this paper, we generate images under \texttt{PromptSet} by Stable Diffusion v1.5-Base \citep{ramesh2022hierarchical}. Finally, without specification, we use 50 steps DDIM sampling \citep{song2020denoising}. 
	%	We use the Stable Diffusion v1.5 \citep{ramesh2022hierarchical} to generate 1600 images under these text prompts \footnote{In the main part of this paper, we explore the text prompt follows the given template. The experiments under complicated text prompts are in Appendix \ref{app:Complex Text Prompts}, and the results are similar to the ones under simple text prompts.}.
	\par
	From \citep{tang2023daam}, though stable diffusion generates encoded VQ image latents \citep{esser2021taming}. These latents preserve semantic information transformed by text prompt through cross-attention module \eqref{eq:cross attention}. Notably, in the cross-attention module, the pixel is a weighted sum of token embedding with cross-attention map $\mathrm{Softmax}(QK^{\top} / \sqrt{d})$ as weights. The weights reveal the semantic information of token, as they are the correlations between image query $Q$ and textual key $K$. To check the correlation, we visualize the averaged cross-attention map over all layers of model $\beps_{\btheta}$ under different time steps $t$, from 50 to 1. 
	\par
	Interestingly, the cross-attention map of each token already has a semantic shape in the early stage of the denoising process, e.g., for $t = 40$ in example Figure \ref{fig:cross-attention map}. This can hold only if the overall shape of the image is constructed in this early stage, so that each pixel can correspond to the correct token. To further investigate this, we compute the average cross-attention map of each token under the aforementioned \texttt{PromptSet}. We compare the shape of the cross-attention map and the final generated image quantitatively by transforming them into canny images \citep{gonzales1987digital} and computing the F1-score \citep{fscore} ($\text{F1}_{t}$ for each $t$) between these canny images. To compare the difference over different time steps more clearly, we plot the relative F1-score $\text{F1}_{t} / \text{F1}_{1}$ ($t=1$ the image has been recovered). The result in Figure \ref{fig:cross-attention map ratio} shows the shape of the cross-attention map rapidly close to the ones of the generated image in the early stage of denoising, which is consistent with our speculation and the result in \citep{zhang2024cross}, where they conclude that the cross-attention map will converge during the denoising process.  
	% That says, during the denoising process, the shape of cross-attention map rapidly converged to the shape of the final stage.             
	\subsection{Frequency Analysis}

	\begin{figure*}[t!]
		\centering
		\vspace{-0.5in}
		\subfloat[Noisy data and its high, low frequency parts\label{fig:corrupted data}]{\includegraphics[scale=0.25]{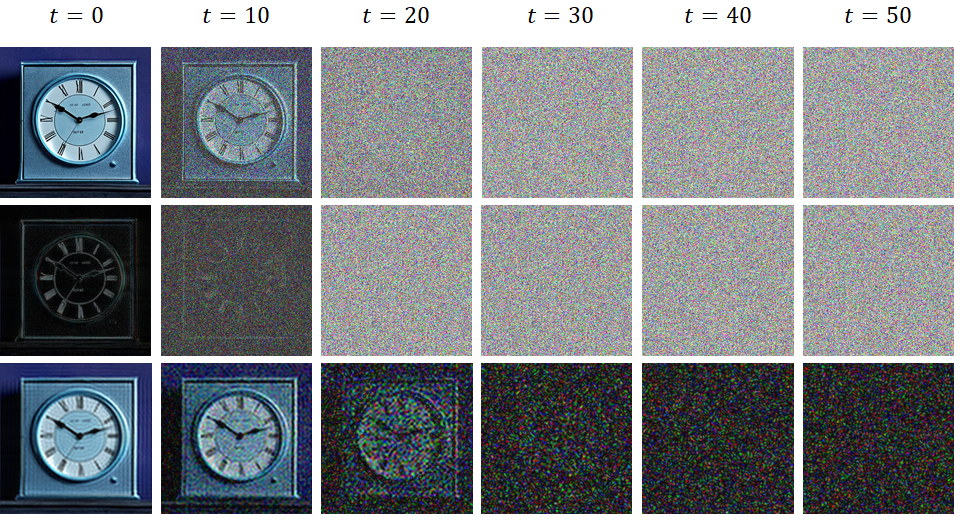}}
		\subfloat[Norm of features  $\sqrt{\baralpha_{t}}\bx_{0}$ and $\sqrt{1 - \baralpha_{t}}\beps_{t}$\label{fig:norm}]{\includegraphics[scale=0.22]{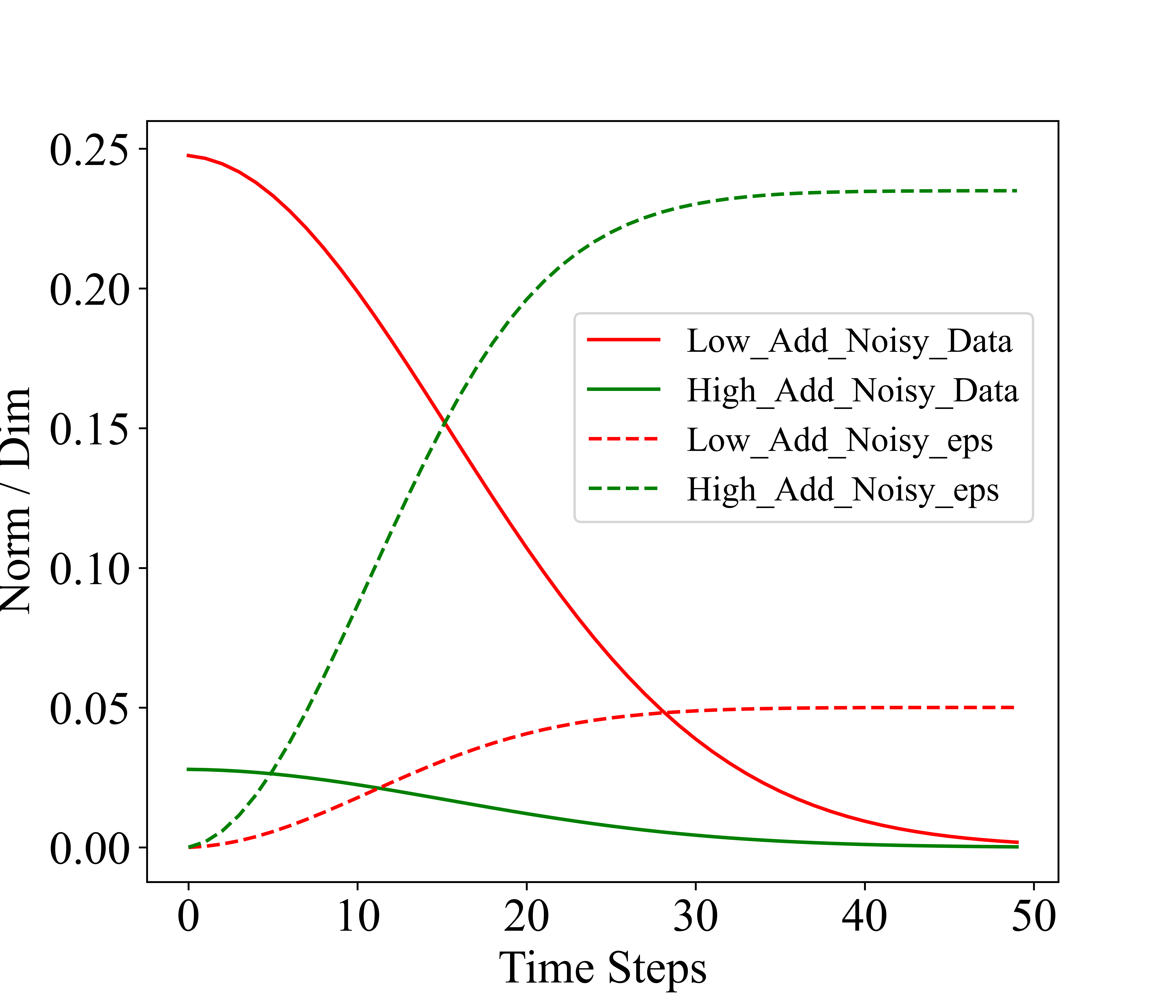}}
		\subfloat[Ratio of high / low frequency parts variation\label{fig:ratio}]{\includegraphics[scale=0.22]{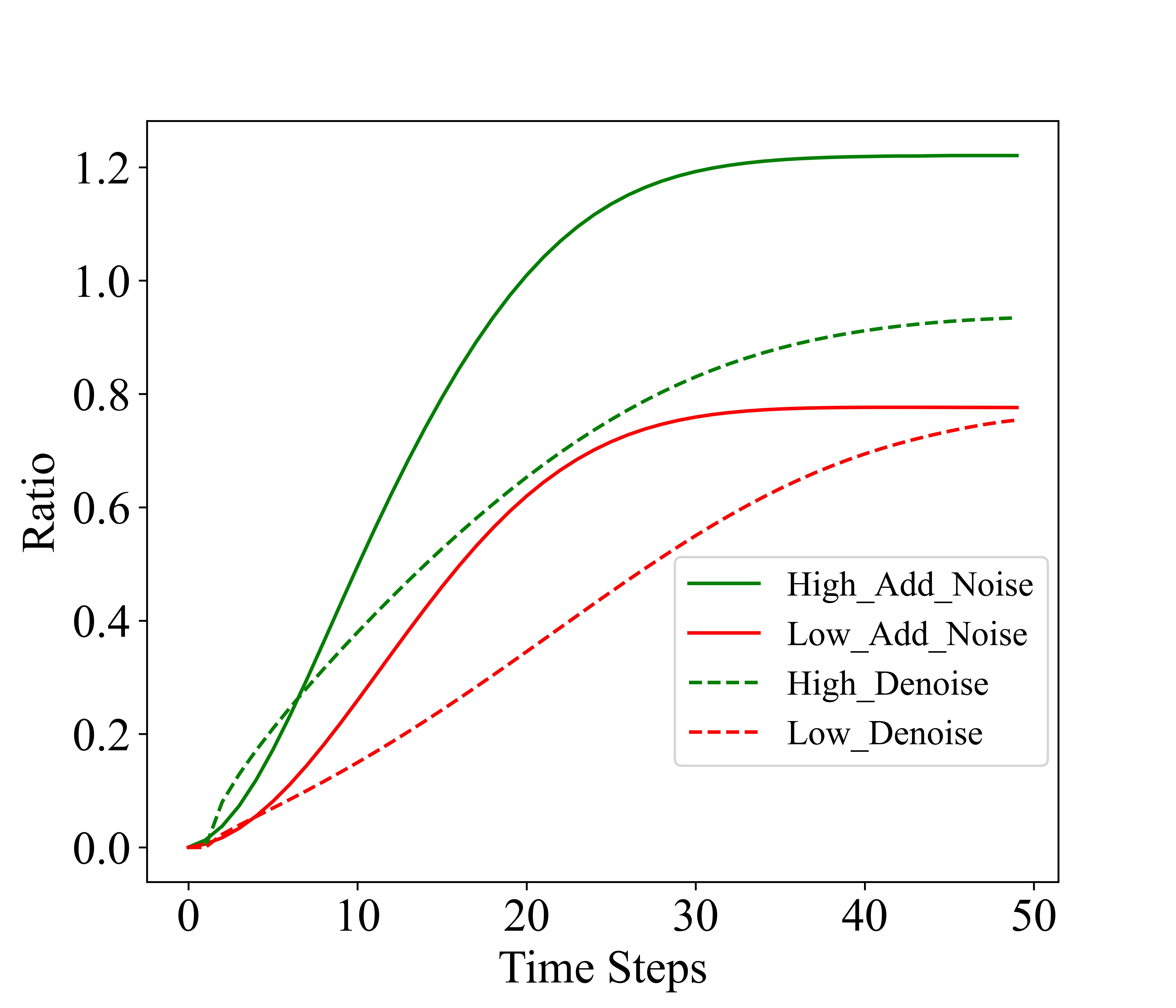}}
		
		\caption{Figure \ref{fig:corrupted data} visualizes the completed noisy data and its high-frequency, and low-frequency parts over different time steps, listed from top to bottom. Figures \ref{fig:norm} and \ref{fig:ratio} measure the low/high-frequency signals of $\bx_{t}$. In Figure \ref{fig:norm}, ``Low\_Add\_Noisy\_Data/eps'' means the norm of $\sqrt{\baralpha_{t}}\bx_{0}^{\rm low}$ and $\sqrt{1 - \baralpha_{t}}\beps_{t}^{\rm low}$, vise versa for ``High\_...''. On the other hand, Figure \ref{fig:ratio} measures the variation ratio of high/low frequency parts of images during the noising/denoising process. For example, ``High\_Add\_Noise'' represents $\|\bx_{t}^{\rm high} - \bx_{0}^{\rm high}\| / \|\bx_{0}^{\rm high}\|$ during noising process.}
		%		\vspace{-0.1in}
	\end{figure*}
	% \begin{figure}[t!]
		% 	\centering
		% 	\includegraphics[scale=0.37]{./pic/frequency/norm_vs_ratio.pdf}
		% 	\caption{Our results are measured under the image latent $\bx_{t}$. The left and right figures are respectively the averaged norm and ratio. For averaged norm, the solid line and dash line are respectively corresponded to  $\sqrt{\baralpha_{t}}\bx_{0}$ and $\sqrt{1 - \baralpha_{t}}\beps_{t}$. The ratio is measured under $\bx_{t}$.}
		% 			% \vspace{-0.2in}
		% 	\label{fig:norm_vs_ratio}
		% \end{figure}
	To further explain the above phenomenon, we refer to the frequency-signal analysis. It has been well explored that the low-frequency signals represent the overall smooth areas or slowly varying components of an image (related to the overall shape). On the other hand, the high-frequency signals correspond to the fine details or rapid changes in intensity (related to attributes like textures) \citep{gonzales1987digital}. Thereafter, to explain the ``first overall shape then details'' in the denoising process, it is natural to refer to the variations in frequency signals of images during the denoising process. 
	\par 
	Mathematically, suppose the clean data (image latent) $\bx_{0}$ has $M\times N$ dimensions for each channel with $\bx_{t}$ defined in \eqref{eq:noisy data}. Then the Fourier transformation $F_{\bx_{t}}(u, v)$ (with $u \in [M], v\in[N]$) of $\bx_{t}$ is 
	\begin{equation}\label{eq:frequency of xt}
		\begin{aligned}
			F_{\bx_{t}}(u, v) & = \frac{1}{MN}\sum\limits_{k = 0}^{M - 1}\sum\limits_{l = 0}^{N - 1}\bx_{t}^{kl}\exp\left(-2\pi \ri\left(\frac{ku}{M} + \frac{lv}{N}\right)\right) \\
			& = \sqrt{\baralpha_{t}}F_{\bx_{0}}(u, v) + \sqrt{1 - \baralpha_{t}}F_{\beps_{t}}(u, v),
		\end{aligned}
		\small
	\end{equation}
	where $\sqrt{-1} = \ri$, and $\bx_{t}^{kl}$ is the $(k, l)$ component of $\bx_{t}$. As we do not now the distribution of $\bx_{0}$, we explore the $F_{\beps_{t}}(u, v)$ in sequel. The result is in the following proposition proved in Appendix \ref{app:proofs of proposition}.
	\begin{restatable}{proposition}{frequency}\label{pro:frequency}
		For all $u\in[M], v\in[N]$, with high probability, the complex number $F_{\beps_{t}}(u, v)$ satisfies 
		\begin{equation}
			\small
			\left\|F_{\beps_{t}}(u, v)\right\|^{2} \approx \cO\left(\frac{1}{MN}\right). 
		\end{equation}
	\end{restatable}
	This proposition indicates that under large image size ($MN$), the strength of frequency signals (no matter low or high) of standard Gaussian are equally close to zero. Thus, the frequency signal of $\bx_{0}$ in noisy data $\bx_{t}$ is mainly corrupted by the shrink factor $\baralpha_{t}$ due to \eqref{eq:frequency of xt}, instead of the noise in it. 
	\par
	However, as visualized in Figure \ref{fig:corrupted data} \footnote{Please note that the $t$ in Figure \ref{fig:corrupted data} and the other parts of this paper refers to the corresponding $t$-th time step of 50-steps DDIM sampling, which corresponds to $20t$ steps in \citep{ho2020denoising}}, in contrast to high-frequency part of image, the image's low-frequency parts \footnote{We distinguish low frequency by threshold 20\%, i.e., the lowest 20\% parts of spectrum are low-frequency.} are more robust than the ones of high-frequency. For example, for $t=20$ in Figure \ref{fig:corrupted data}, the shape of the clock is still perceptible in the low-frequency part of the image. \footnote{This fact is also observed in \citep{si2023freeu,yang2023diffusion}, while they leverage this phenomenon to refine the architecture of UNet, and the rationale behind the phenomenon is unexplored.} If this fact is generalized to image latent, then it explains the two stages of generation as observed in Section \ref{sec:Two Stages of Denoising Process}. Because the low-frequency parts are not corrupted until the end of the adding noise process. Then, it will be recovered at the beginning of the reverse denoising process. 
	\par
	To investigate this, in Figures \ref{fig:norm} and \ref{fig:ratio}, we plot the averaged results over time steps of variation of low/high-frequency parts in images generated by \texttt{PromptSet}. In these figures, $\bx_{t}^{\rm low}$ is the low-frequency part of $\bx_{t}$ and  vice-versa for high-frequency part $\bx_{t}^{\rm high}$. 
	%	of $\sqrt{\baralpha_{t}}\bx_{0}$ and $\sqrt{1 - \baralpha_{t}}\beps_{t}$ over various time steps. To make the comparison more clear, for each noisy data $\bx_{t}$, we plot the variants of ratio $\|\bx_{t}^{\rm low} - \bx_{0}^{\rm low}\| / \|\bx_{0}^{\rm low}\|$ and $\|\bx_{t}^{\rm high} - \bx_{0}^{\rm high}\| / \|\bx_{0}^{\rm high}\|$ over $t$, respectively. Here $\bx_{t}^{\rm low}$ is the low-frequency part of $\bx_{t}$ and vice-versa for high-frequency $\bx_{t}^{\rm high}$. The results are obtained by taking average over the data generated under \texttt{PromptSet}, and summarized in Figure \ref{fig:norm_vs_ratio}. 
	As can be seen, in Figure \ref{fig:ratio}, the behavior of $\bx_{t}$ is similar under add/de noise processes, and the reconstruction of low-frequency signals is faster than the high-frequency signals. On the other hand, by comparing ``Low\_...\_Data'' ( $\|\sqrt{\baralpha_{t}}\bx_{0}^{\rm low}\|$) and ``High\_...\_Data'' ($\|\sqrt{\baralpha_{t}}\bx_{0}^{\rm high}\|$) in Figure \ref{fig:norm}, we observe the strength of high-frequency signals are significantly lower than the low-frequency signals, which seems to be a property adopted from natural image \citep{gonzales1987digital}. However, the relationship oppositely holds for Gaussian noise, which is implied by Proposition \ref{pro:frequency}, as the frequency signals of noise $\beps_{t}$ under each spectrum are all close to zero, while the high-frequency parts contain 80\% spectrum, so that $\beps_{t}^{\rm high}$ is larger than the $\beps_{t}^{\rm low}$. 
	\par
	These observations explain the phenomenon ``first overall shape then details''. Since the low-frequency parts of the image (decide overall shape) are not totally corrupted until the end of the noising process. Thus, they will be firstly recovered during the reverse denoising process, while the phenomenon does not hold for low-frequency parts of the image, as they are quickly corrupted during the noising process, so they will not be recovered until the end of denoising.   
	
	\section{The Working Mechanism of Text Prompt}\label{sec:The Working Mechanism of Text Prompt}
	We have verified that the denoising process has two stages i.e., ``first overall shape then details''. Next, we explore the working mechanism of text prompts during these stages. Our main observations are two fold, 1): The special tokens [\texttt{EOS}] dominate the influence of text prompt. 2): The text prompt mainly works on the first overall shape reconstruction stage of the denoising process. 
	\subsection{[\texttt{EOS}] Contains More Information}\label{sec:eos contains more information}
	In T2I diffusion, the text prompt is encoded by auto-regressive CLIP text encoder, with semantic tokens (SEM) enclosed with special tokes [\texttt{SOS}] and [\texttt{EOS}]. For such three classes of tokens, as the information in these tokens is conveyed by the cross-attention module, we first compute the averaged weights over pixels in the cross-attention map for each class. The weights are computed by taking the average over \texttt{PromptSet} and presented in Figure \ref{fig:token_cross_atten}. As can be seen, the weights of [\texttt{SOS}] are significantly larger than the other classes. However, due to the CLIP text encoder is an auto-regressive model, [\texttt{SOS}] does not contain any semantic information. Therefore, we conclude that the influence of [\texttt{SOS}] is mainly adjusting the whole cross-attention map i.e., weights on the other tokens. To further verify this conclusion, we conduct experiments in Appendix \ref{app:eos Substitution}. A similar phenomenon is observed in single-modality LLM \citep{xiao2023efficient}. As the information of text prompt is conveyed by semantic tokens and [\texttt{EOS}], we will focus on them instead of [\texttt{SOS}] in the sequel.    
	\par
	As both SEM and [\texttt{EOS}] contain the semantic information in the text prompt, we first explore which of them has larger impact on T2I generation. To this end, we select 3000 pairs of text prompts from \texttt{PromptSet} (2000 pairs follow the template, the other 1000 pairs have complex prompts), where the two text prompts are represented as ``[\texttt{SOS}] + Prompt $A$ ($B$) + [\texttt{EOS}]$_{A(B)}$''. For each pair, we switch their [\texttt{EOS}] to construct the new text prompt pairs as ``[\texttt{SOS}] + Prompt $A$ ($B$) + [\texttt{EOS}]$_{B(A)}$''. 
	\par
	We examine the generated images under these artificially constructed text prompts (namely \texttt{Switched-PromptSet} (\texttt{S-PromptSet})). We call $A$ from Prompt$_{A}$ as ``source'' and $B$ from [\text{EOS}]$_{B}$ as ``target'' for ``[\texttt{SOS}] + Prompt $A$ + [\texttt{EOS}]$_{B}$'', and vice versa. For the generated images under these prompts, we measure their alignments with the source and target prompts, respectively. The used metrics are the three standard ones in measuring text-image alignment: CLIPScore \citep{radford2021learning,hessel2021clipscore}, BLIP-VQA \citep{li2022blip,huang2023t2i}, and  MiniGPT4-CoT \citep{zhu2023minigpt,huang2023t2i} (details are in Appendix \ref{app:Text-Image Alignment Metrics}).    
	%	Therefore, we will examine the impact of these special tokens, as well as the semantic tokens in the text prompt. Note that the textual information is injected by the cross-attention module \eqref{eq:cross attention}, which outputs a weighted sum of token features, as explained in Section \ref{sec:Two Stages of Denoising Process}. We will explore the weights and features of the tokens separately. 
	\par
	
	\begin{figure}[t]
		\vspace{-0.5in}
		\centering
		\includegraphics[scale=0.5]{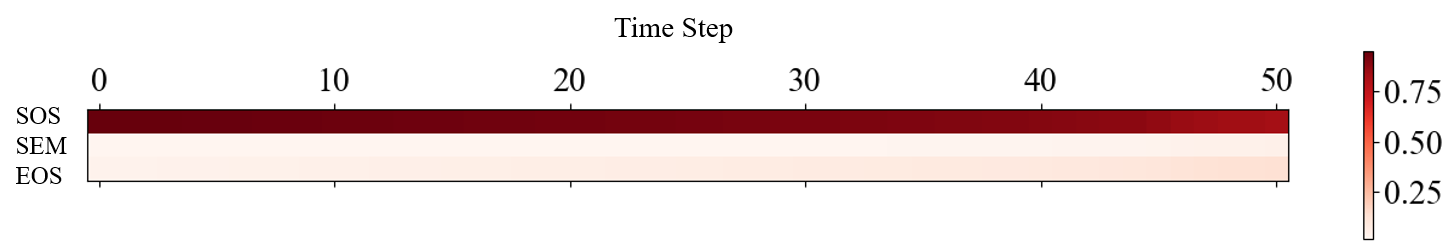}
		\caption{Averaged weights in cross-attention map over pixels of three classes of tokens. For each prompt in \texttt{PromptSet}, the result is obtained by taking average over tokens in each class. The final result is the average over \texttt{PromptSet}. Notably, the weights on [\texttt{SOS}] are all larger than 0.9.}
		\vspace{-0.2in}
		\label{fig:token_cross_atten}
	\end{figure}
	
	\begin{figure}[t!]
		%		\vspace{-0.5in}
		\centering
		\includegraphics[width=0.6\columnwidth]{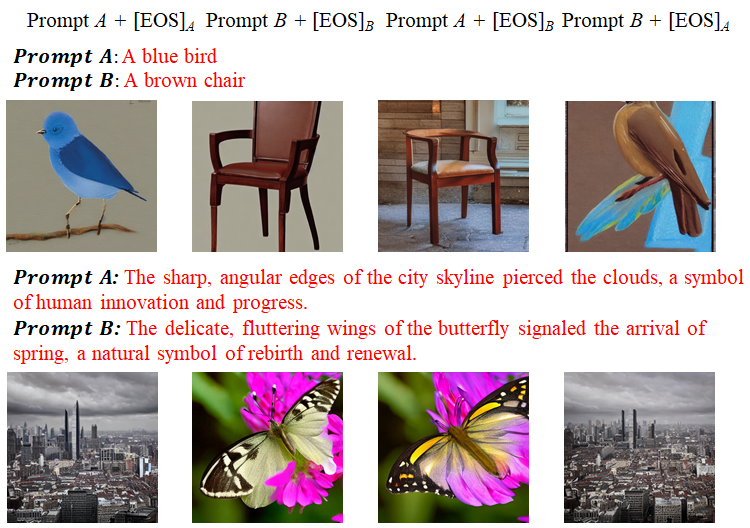}
		\caption{Images under prompts from \texttt{S-PromptSet} with switched [\texttt{EOS}]. The objects are consistent with the ones conveyed by [\texttt{EOS}], while some information in semantic tokens is still conveyed.}
		\label{fig:eos substitbution generated image}
		\vspace{-0.2in}
	\end{figure} 
	
	%	\subsection{[\texttt{EOS}] Decides the Overall Information}\label{sec:Decides the Overall Information}
	
	%	In this subsection, we explore the impact of the  special token [\texttt{EOS}]. It has been shown in Section \ref{sec:First Overall Shape then Details} that the overall shape is firstly decided during the denoising process. The goal of this subsection is to explore the relationship between this phenomenon with [\texttt{EOS}], as it contains more information than semantic tokens.  
	The results are in Table \ref{tbl:eos substitution alignment}. Surprisingly, the generated images under the constructed text prompts are more likely to be aligned with the target prompt instead of the source prompt. That says, even with prefixed irrelevant semantic tokens, the information contained in [\texttt{EOS}] dominates the denoising process (especially for overall shape) as in Figure \ref{fig:eos substitbution generated image}. Thus, we conclude that the special tokens [\texttt{EOS}] have a larger impact than semantic tokens in prompt during T2I generation. We have two speculations about this phenomenon. 1): Owing to the auto-regressive encoded text prompt, unlike semantic tokens, [\texttt{EOS}] contains complete textual information, so that it decides the pattern of the generated image. 2): The number of [\texttt{EOS}] is usually larger than semantic tokens, as the prompt is enclosed by [\texttt{EOS}] to length 76. An ablation study in Appendix \ref{app:Ablation Study on Nubmer of eos} verifies this speculation.
	\par
	In summary, our conclusion in this subsection can be summarized as: \textbf{\emph{In T2I generation, the special token [\texttt{EOS}] decides the overall information (especially shape) of the generated image.}}
	\begin{remark}
		For the generated images under \texttt{S-PromptSet}, we find some information in semantic tokens is also conveyed, especially for the attribute information in it, e.g., ``brown'' color in the last image of the first row in Figure \ref{fig:eos substitbution generated image}. We explore this in Appendix \ref{sec:Conveyed Information in Semantic Tokens} and explain this as: unlike noun information, attributes in semantic tokens may not conflict with the contained information in [\texttt{EOS}] (which quickly decides the overall shape of the generated image), so that has potential to be conveyed.      
	\end{remark}
	\begin{figure*}[t!]
		\centering
		\vspace{-0.5in}
		\includegraphics[scale=0.5]{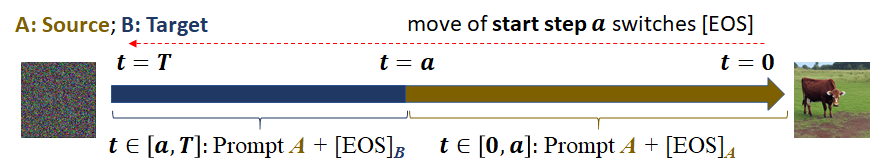}
		%    	 \vspace{-0.2in}
		\caption{Desnoising process under text prompt with switched [\texttt{EOS}] in $[a, 50]$.}
		\label{fig:switches eos}
		\vspace{-0.2in}
	\end{figure*}
	\begin{figure}[t!]
		\begin{minipage}[t!]{0.33\linewidth}
			\centering
			%		\vspace{-0.1in}
			%		\vspace{-0.1in}
			\captionof{table}{The alignment of generated image with its source and target prompts. The prompts are constructed with switched [\texttt{EOS}].}
			\label{tbl:eos substitution alignment}
			% \vspace{0.1in}
			\scalebox{0.7}{
				\begin{tabular}{l|cc}
					\hline
					\diagbox{Alignment}{Prompt} &  Source &  Target            \\
					\hline
					Text-CLIPScore $\uparrow$ &  0.2363  &  \textbf{0.2758} \\
					BLIP-VQA$\uparrow$  &     0.3325      &   \textbf{0.4441}   \\
					MiniGPT-CoT$\uparrow$  &  0.6473           &\textbf{0.7213}  \\
					\hline
			\end{tabular}}
		\end{minipage}
		\begin{minipage}[t!]{.33\textwidth}
			\includegraphics[scale=0.28]{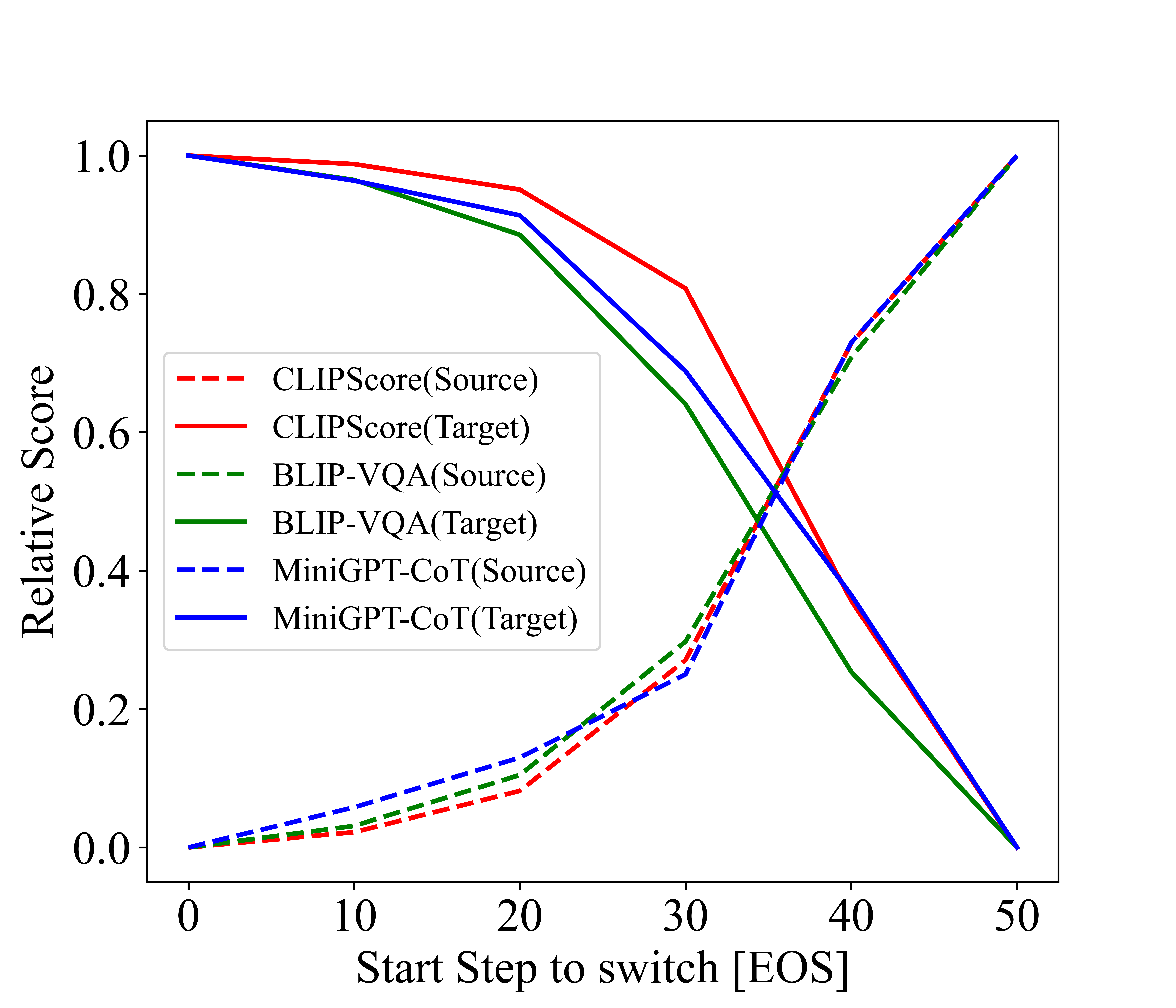}
		\end{minipage}
		\begin{minipage}[t!]{.33\textwidth}
			\caption{Relative text-image alignments (``current minus worst'' over ``best minus worst'') with source (or target) prompt under switched [\texttt{EOS}] (substitution in [Start Step, 50], Figure \ref{fig:switches eos}). Alignments with source and target prompts are respectively solid and dot lines.}
			\label{fig:eos_50_steps}
		\end{minipage}
		\vspace{-0.2in}
	\end{figure}
	
	\subsection{The Text Prompt Mainly Working on the First Stage}\label{sec:the influence of text prompt}
	In Section \ref{sec:Two Stages of Denoising Process}, we have conclude that the denoising process is divided into two stages ``first overall shape then details''. Next, we explore the relationships between text prompts and the two stages. 
	We start with special tokens [\texttt{EOS}] which contain major information in T2I generation. During the whole 50 denoising steps of T2I generation under prompts from \texttt{S-PromptSet}, we vary the starting point of substituting [\texttt{EOS}] i.e., the used text prompt is ``[\texttt{SOS}] + Prompt $A$ + [\texttt{EOS}]$_{B}$ (resp. [\texttt{EOS}]$_{A}$) for $t\in [\text{Start Step}, 50]$ (resp. $t\in [0, \text{Start Step}]$) with $\text{``Start Step''}\in [0, 50]$, i.e., Figure \ref{fig:switches eos}. We compare the alignments of generated images with source / target prompts as in Figure \ref{fig:eos_50_steps}. 
	\par
	In Figure \ref{fig:eos_50_steps}, the alignment with the target prompt slightly decreases, until the ``Start Step'' of substitution close to 50. This shows that the information in [\texttt{EOS}] has been conveyed during the first few steps of the denoising, which is the overall shape reconstruction stage according to Section \ref{sec:First Overall Shape then Details}.
	\par
	Following the revealed working stage of \texttt{[EOS]}, we explore whether the whole text prompt also works in this stage. If so, the T2I generation will only depends on $\beps_{\btheta}(t, \bx_{t}, \emptyset)$ in \eqref{eq:noise prediction} for small $t$. To see this, we vary the $w$ in \eqref{eq:noise prediction} to control the injected information from the text prompt during the denoising process. Concretely, for $a$ as the starting step of removing text prompt, i.e., during $t\in [0, a)$, we use $w=7.5$, and $w=0$ for $t\in [a, 50]$, where $a\in[0, 50]$. Then, the text prompt only works for $t \in [0, a)$. We generate target images $\bx_{0}^{50}$ under \texttt{PromptSet} with standard denoising process ($a=50$), and compare them with the ones $\bx_{0}^{a}$ generated under varied $a\in[0, 50]$ (Figure \ref{fig:inject text prompt test}). The image-image alignments are measured by standard metrics CLIPScore and $L_1$-distance \citep{gonzales1987digital}. To eliminate magnitude, we report relative results, i.e., ``current minus worst'' over ``best minus worst''. 
	% i.e., $m(\bx_{0}^{a}, \bx_{0}^{50}) / m(\bx_{0}^{0}, \bx_{0}^{50})$, where $m(\cdot, \cdot)$ is one of the two metrics, and the maximal distance is $m(\bx_{0}^{0}, \bx_{0}^{50})$. 
	\par
	The results are in Figure \ref{fig:text_prompt_influence}. During generation, the text information is absence for $t\in [a, 50]$, while Figure \ref{fig:text_prompt_influence} indicates that alignments between $\bx_{0}^{a}$ and target $\bx_{0}^{50}$ will quickly be small only for large $a$ (from 30 to 50). This shows that only if removing the textual information under large $t$, its influence to generated image is removed. Therefore, we can conclude: \textbf{\emph{The information of text prompt is conveyed during the early stage of denoising process.}} Therefore, the overall shape of generated image is mainly decided by the text prompt, while the its details are then reconstructed by itself.   
	
	%    In this way, the text prompt guidance is unexist in $(a, 50]$ steps. We generate images conditioned on \texttt{PromptSet}, under the above $w$ schedule. We compare the image-image alignment of these images with the ones generated under $w=7.5$ for all 50 steps, by .
	
	%    The results are summarized in Figure \ref{fig:text_prompt_influence}. As can be seen, the alignment metrics are rapidly converge to their extreme values, then become flat around a constant after a large $a$. Moreover, the generated images after $a \geq 10$ in Figure \ref{fig:text_prompt_influence} are also visually similar to the images generated under $w=7.5$ for the whole denoising process. These observations are all positively answer the aforementioned question. That says: \emph{The information of text prompt are already conveyed during the early stage of denoising process, and the following reconstruction relies on images themselves.} 
	
	\paragraph{Discussion.} Next, let us explain the phenomenon. Technically, in \eqref{eq:noise prediction}, the $\beps_{\btheta}(t, \bx_t, \cC, \emptyset)$ was proposed to approximate  $\nabla_{\bx}\log{p_{t}(\bx_{t}\mid \cC)} / \sqrt{1 - \baralpha_{t}}$ with decomposition ($p_{t}$ is the density of $\bx_{t}$) 
	%    can be rewritten as 
	%    \begin{equation}\label{eq:noise prediction decomp}
		%        \small
		%            \beps_{\btheta}(t, \bx_t, \cC, \emptyset) = \beps_{\btheta}(t, \bx_t, \emptyset) + w\left(\beps_{\btheta}(t, \bx_t, \cC) - \beps_{\btheta}(t, \bx_t, \emptyset)\right).
		%    \end{equation}
	\begin{equation}\label{eq:conditional score}
		\small
		\nabla_{\bx}\log{p_{t}(\bx_{t}\mid \cC)} = \nabla_{\bx}\log{p_{t}(\bx_{t})} + \nabla_{\bx}\log{p_{t}(\cC\mid \bx_{t})}.
	\end{equation}
	Comparing \eqref{eq:noise prediction} and \eqref{eq:conditional score}, it holds $\beps_{\btheta}(t, \bx_{t}, \emptyset)\propto \nabla_{\bx}\log{p_{t}(\bx_{t})}$ \footnote{This can be verified by the training strategy of it in \citep{ramesh2022hierarchical}, where $\beps_{\btheta}(t, \bx_{t}, \emptyset)$ is used to predict noise in noisy data without condition injected.} and $w(\beps_{\btheta}(t, \bx_{t}, \cC)x - \beps_{\btheta}(t, \bx_{t}, \emptyset)) \propto \nabla\log{p_{t}(\cC\mid \bx_{t})}$. From \citep{song2020score}, the denoising process \eqref{eq:forward propogation} aims to maximize log-likelihood $\log{p_{0}(\bx_{0}\mid \cC)}$. Then, moving along the direction $\nabla_{\bx}\log{p_{t}(\cC\mid \bx_{t})}$ (leads to large $\log{p_{t}(\cC\mid \bx_{t})}$) push $\bx_{t}$ to be aligned with the text prompt $\cC$ during the decreasing of $t$. Adding such a moving direction is standard in conditional generation \citep{marek2021oodgan,song2020denoising,farid2023latent,madaan2021generate}. As shown in Figure \ref{fig:noise_norm}, during the denoising process, $\bx_{t}$ will gradually to be consistent with $\cC$, so that $\nabla\log{p_{t}(\cC\mid \bx_{t})}$ will decrease with $t$. Thereafter, we observe the impact of text prompt conveyed by this term decreases with $t\to 0$. Notably, owing to the quickly reconstructed overall shape of image in Section \ref{sec:First Overall Shape then Details}, the generated $\bx_{t}$ will quickly be consistent with $\cC$, so that explain the quickly decreasing $\nabla\log{p_{t}(\cC\mid \bx_{t})}$. 
	
	\begin{remark}
		In this section, we verify the injected textual information are all conveyed in the first stage of diffusion process. In fact, this phenomenon is also generalized to the other types of information, e.g., conditional image information in subject-driven generation \citep{chen2024anydoor,li2024Photomaker}, we verify this in Appendix \ref{app:subject}
	\end{remark}
	%    
	%    To explain this phenomenon, since we have observed that the overall shape of generated image is quickly reconstructed, the 
	\begin{figure*}[t!]
		\centering
		%    	\vspace{-0.5in}
		\includegraphics[scale=0.5]{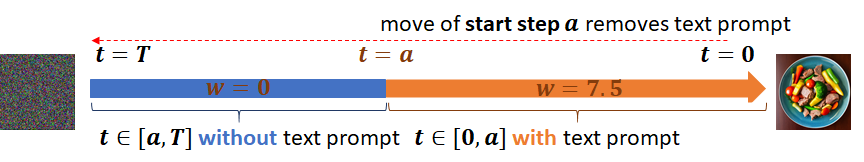}
		%    	 \vspace{-0.2in}
		\caption{Desnoising with text prompt injected in $[0, a]$.}
		\label{fig:inject text prompt test}
		\vspace{-0.1in}
	\end{figure*}
	
	\begin{figure*}[t!]
		%		\centering
		%    		\vspace{-0.5in}
		\subfloat[Relative Image-Alignment \label{fig:text_prompt_influence}]{\includegraphics[scale=0.33]{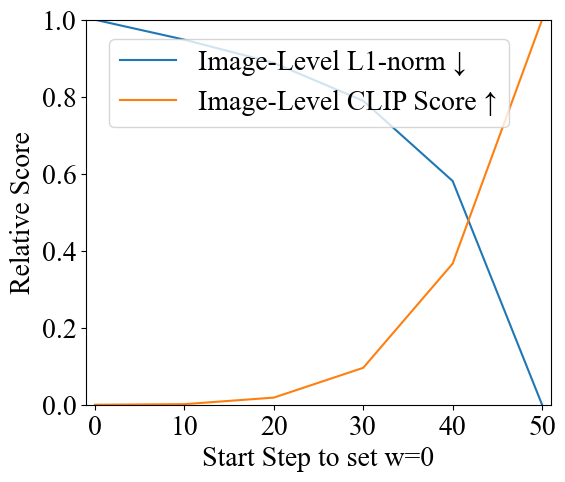}}
		% \vspace{-0.2in}
		\subfloat[Norm/Dim of (Un)/Conditional Model\label{fig:noise_norm}]{\includegraphics[scale=0.33]{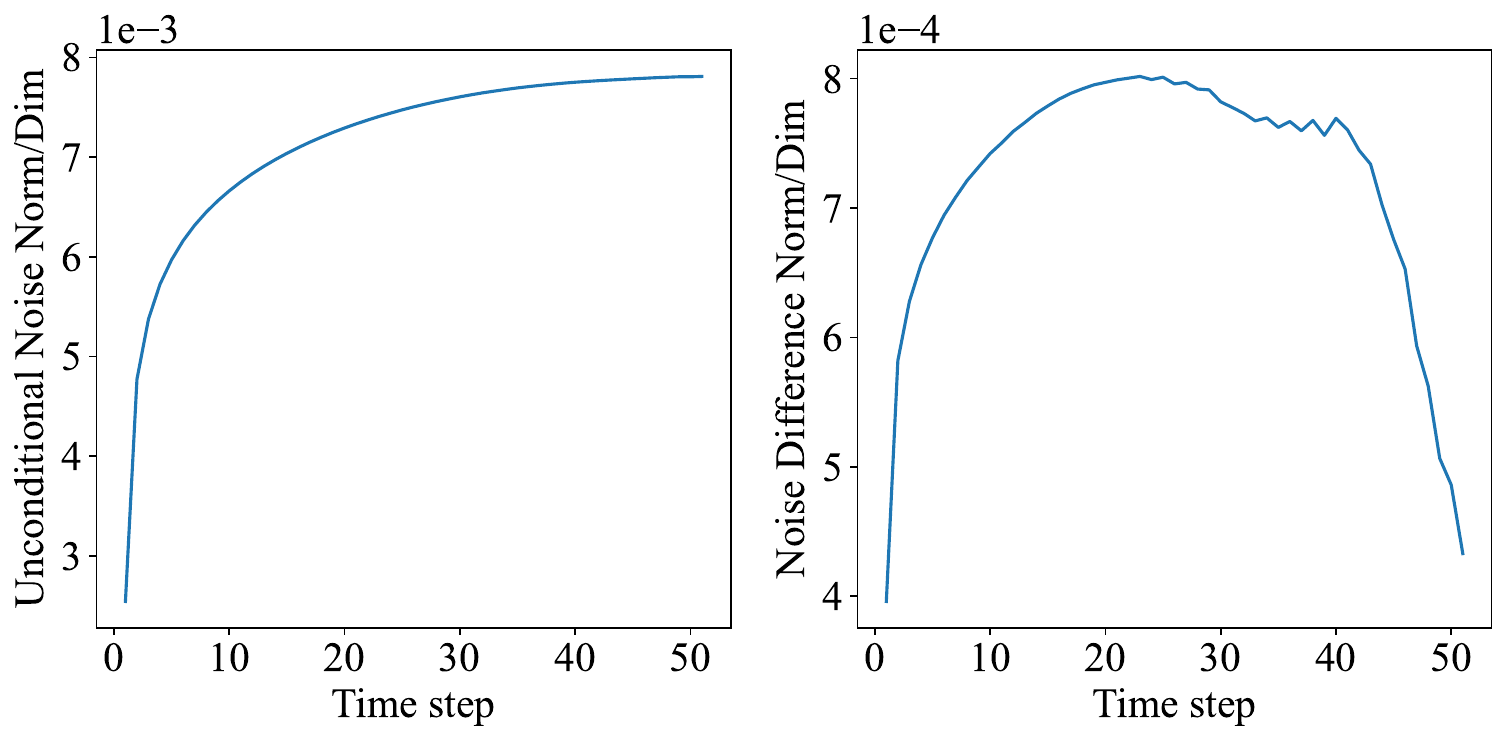}}
		\caption{Figure \ref{fig:text_prompt_influence} is the relative difference ``current minus worst'' over ``best minus worst'' under different start step $a$ of Denoising process Figure \ref{fig:inject text prompt test}. The last two figures \ref{fig:noise_norm} are per-dimensional norm of unconditional noise $\beps_{\btheta}(t, \bx_t, \emptyset)$ and noise difference $ w\left(\beps_{\btheta}(t, \bx_t, \cC) - \beps_{\btheta}(t, \bx_t, \emptyset)\right)$}
		\vspace{-0.1in}
	\end{figure*}
	
	\section{Application}\label{sec:application}
	%    In this section, we give an application of the conclusion that the text prompt mainly works on the first stage of the denoising process. 
	\paragraph{Acceleration of Sampling.} Since the information contained in text prompt is mainly conveyed by the noise prediction with condition $\beps_{\btheta}(t, \bx_{t}, \cC)$, we can consider removing the evaluation of after the first few steps of denoising process. This is because the information in text prompt has been conveyed in this stage, and the computational cost can be significantly reduced without evaluating $\beps_{\btheta}(t, \bx_{t}, \cC)$. 
	\par
	Therefore, we substitute the noise prediction $\beps_{\btheta}(t, \bx_{t}, \cC,  \emptyset)$ as 
	\begin{equation}\label{eq:new noise prediction}
		\small
		\beps_{\btheta}(t, \bx_{t}, \cC, \emptyset) = 
		\begin{dcases}
			\beps_{\btheta}(t, \bx_t, \emptyset) + w\left(\beps_{\btheta}(t, \bx_t, \cC) - \beps_{\btheta}(t, \bx_t, \emptyset)\right) \qquad & a\leq t; \\
			\beps_{\btheta}(t, \bx_t, \emptyset) \qquad & 0 \leq t < a. 
		\end{dcases}
	\end{equation} 
	By varying $a\to T$ in \eqref{eq:new noise prediction}, the inference cost is reduced as an evaluation of $\beps_{\btheta}(t, \bx_t, \cC)$ is saved.
	\par
	To evaluate the saved computational cost of using noise prediction \eqref{eq:new noise prediction} during inference and the quality of generated data, we consider applying it on two standard samplers DDIM \citep{song2020denoising} and DPM-Solver \citep{lu2022dpm} on a benchmark dataset MS-COCO \citep{lin2014microsoft} in T2I generation. We consider backbone models Stable-Diffusion (SD) v1.5-Base, SD v2.1-Base \citep{ramesh2022hierarchical}, and Pixart-Alpha \citep{chen2024pixart}. Concretely, we apply noise prediction \eqref{eq:new noise prediction} with varied $a$ to generate 30K images from 30K text prompts in the test set of MS-COCO, for each sampler and backbone model. We compare the difference (measured by $L_{1}$-distance and Image-Level CLIPScore) between the generated images under $a > 0$ and $a = 0$ (the standard noise prediction). The results are in Table \ref{tbl:acceleration}, where we also report the Frechet Inception Distance (FID) score \citep{heusel2017gans} under each $a$ to evaluate the quality of generated images.  
	
	\begin{figure*}[t!]
		\centering
		%		\vspace{-0.5in}
		\includegraphics[scale=0.5]{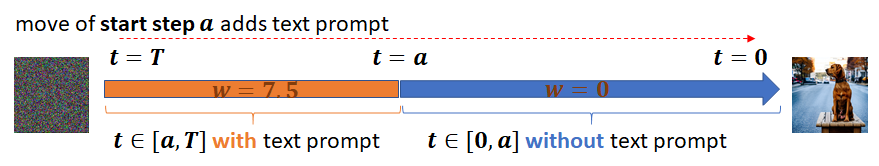}
		% \vspace{-0.2in}
		\caption{Desnoising under $\beps_{\btheta}$ \eqref{eq:new noise prediction}. The text prompt is injected in $[a, T]$, instead of $[0, a]$ in Figure \ref{fig:inject text prompt test}.}
		\label{fig:inject text prompt}
		%		\vspace{-0.2in}
	\end{figure*}
	
	\begin{figure*}
		\centering
		\includegraphics[scale=0.5]{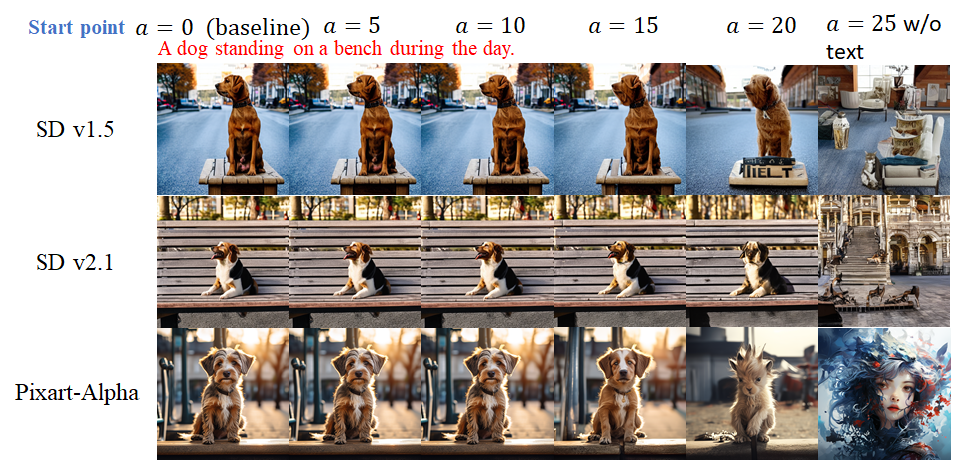}
		\caption{The generated images with 25 steps DPM-Solver under $\beps_{\btheta}$ in \eqref{eq:new noise prediction} (Figure \ref{fig:inject text prompt}). The textual information is removed during $t\in[0, a]$. With $a\to 25$, the inference cost is decreased.} 
		\label{fig:substitution}
		%		\vspace{-0.2in}
	\end{figure*}
	\par
	The Table \ref{tbl:acceleration} indicates that proper $a$ in \eqref{eq:new noise prediction} significantly reduces the computation cost during the inference stage without deteriorate the quality of generated images. For example, SD v1.5 with $a = 20$ saves 27+\% computational cost, but generates images close to the baseline method ($a = 0$).

	\begin{table*}[t!]
		\caption{The difference between images generated under varied $a$ with the ones of $a=0$. The results are averaged over 30K generated images, and saved latency is evaluated on one V100 GPU.}
		%		\vspace{-0.1in}
		\label{tbl:acceleration}
		%	 \vspace{0.1in}
		\centering
		\scalebox{0.62}{
			{
				\begin{tabular}{l|cccccc|cccccc}
					\hline
					Start point $a$   & 0 (baseline)  & 10 & 20 &  30  & 40  & 50 & 0 (baseline) &  5  &  10 & 15 & 20 &25\\
					\hline 
					Sampler / Backbone        & \multicolumn{6}{c}{DDIM / SD v1.5}    & \multicolumn{6}{c}{DPM-Solver / SD v1.5} \\
					\hline
					Image-CLIPScore $\uparrow$ &1.000     &0.998   &0.996   &0.971  &0.838 & 0.539 &1.000 &0.999 &0.994 &0.956 &0.798 & 0.533\\
					$L_{1}$-distance $\downarrow$  &0.000  &0.011   &0.022  &0.043 &0.087 & 0.195 &0.000 &0.015 &0.027 &0.050 &0.100  & 0.188\\
					Saved Latency $\uparrow$  &0.00\%      &7.84\%   &18.10\%   &27.18\%  &35.24\% & 48.47\% & 0.00\% &8.36\% &17.95\% &26.11\%  &34.86\%  & 47.60\% \\
					FID $\downarrow$  &13.772    &13.770   &13.805   &14.012  &15.048 & 19.296 &14.297 &14.286 &14.725 &14.985 &15.860  & 19.758\\
					Text-CLIPScore $\uparrow$ &31.040    &31.001   &30.894   &30.493  &28.172 & 16.682 &30.992 &30.891 &30.772 &30.205 &26.843 & 16.721\\
					\hline
					Sampler / Backbone        & \multicolumn{6}{c}{DDIM / SD v2.1}    & \multicolumn{6}{c}{DPM-Solver / SD v2.1} \\
					\hline
					Image-CLIPScore $\uparrow$ &  1.000   &0.999   &0.998   &0.986  &0.902 & 0.550 &1.000 &0.999 &0.998 &0.988 &0.901 & 0.543 \\
					$L_{1}$-distance $\downarrow$  &0.000  &0.017   &0.041  &0.077 &0.152 &  0.386 &0.000 &0.026 &0.046 &0.083   &0.160 & 0.369\\
					Saved Latency $\uparrow$  &  0.00\%    &8.68\%   &18.95\%   &28.16\%  &36.19\% & 47.99\% &0.00\% &8.75\% &18.28\% &26.24\% &35.48\% & 47.75\% \\
					FID $\downarrow$  &13.014    &13.011   &13.046   &13.247  &14.242 & 18.472 &13.507 &13.500 &13.914 &14.159 &15.015 & 18.983 \\
					Text-CLIPScore $\uparrow$ &31.413    &31.405   &31.362   &31.113  &29.717 & 16.706 &31.339 &31.326 &31.292 &31.045 &29.653 & 16.639 \\
					\hline
					Sampler / Backbone        & \multicolumn{6}{c}{DDIM / Pixart-Alpha}    & \multicolumn{6}{c}{DPM-Solver / Pixart-Alpha} \\
					\hline
					Image-CLIPScore $\uparrow$ &  1.000   &0.999   &0.935   &0.744  &0.643 & 0.522 &1.000 &0.999 &0.993 &0.911 &0.648 & 0.625\\
					$L_{1}$-distance $\downarrow$  &0.000  &0.024   &0.058  &0.103 &0.169 & 0.247 &0.000 &0.013 &0.022 &0.046   &0.098 & 0.199 \\
					Saved Latency $\uparrow$  &  0.00\%    &8.24\%   &17.98\%   &27.20\%  &34.95\% & 49.15\% &0.00\% &7.92\% &15.77\% &25.18\% &33.80\% & 48.10\% \\
					FID $\downarrow$ &22.651 &22.884 &23.258 &25.485 &29.760 & 36.525 &18.669    &18.520   &18.798   &19.358  &20.494 &   26.159 \\
					Text-CLIPScore $\uparrow$ &28.157    &27.979   &25.719   &19.640  &14.986 & 14.586 &30.733 &30.721 &30.745 &29.416 &20.629 &14.928 \\
					\hline
				\end{tabular}
				%				\vspace{-0.2in}
		}}
	\end{table*}
	
	\section{Conclusion}
	In this paper, we investigate the working mechanism of T2I diffusion model. By empirical and theoretical (frequency) analysis, we conclude that the denoising process firstly constructs the overall shape then details of the generated image. Next, we explore the working mechanism of text prompts. We find its special token \texttt{[EOS]} has a significant impact on the overall shape in the first stage of the denoising process, in which the information in the text prompt is conveyed. Then, the details of images are mainly reconstructed by themselves in the latter stage of generation. Finally, we apply our conclusion to accelerate the inference of T2I generation, and save 25\%+ computational cost.

	\section*{Acknowledgement}
	We gratefully acknowledge the support of Mindspore, CANN(Compute Architecture for Neural Networks) and Ascend AI Processor used for this research. 
	
	%	\newpage
	%        \section*{Broader Impact} 
	%        This paper delves into the functioning mechanism of text prompts in Text-to-Image generation, which has positive applications for society. Specifically, it can foster the advancement of precise and controllable generation, contributing significantly to various applications in real-world (\textit{e.g.}, product image generation, game character and scene generation, advertising image generation, etc). 
	
	\bibliography{reference}
	\bibliographystyle{abbrv}
	
	\newpage
	\appendix
	\input{appendix}
	\newpage

\end{document}

%% file: appendix.tex
\onecolumn
\section{Proofs of Proposition \ref{pro:frequency}}\label{app:proofs of proposition}
\frequency*
\begin{proof}
		Note that $\beps_{t}^{kl}$ (abbreviated as $\beps^{kl}$) are i.i.d. Gaussian random variable for each of $k,l$. Thus we have  
		\begin{equation}
			\small
			\begin{aligned}
				F_{\beps}(u, v)  = \frac{1}{MN}\sum\limits_{k = 0}^{M - 1}\sum\limits_{l = 0}^{N - 1}\beps^{kl}\exp\left(-2\pi \ri\left(\frac{ku}{M} + \frac{lv}{N}\right)\right) = \frac{1}{MN}\sum\limits_{k = 0}^{M - 1}\sum\limits_{l = 0}^{N - 1}\beps^{kl}\exp\left(-\ri\theta^{kl}_{uv}\right),
			\end{aligned}
		\end{equation}
		with $\theta^{kl}_{uv}$ is the $(k, l)$-th angle in complex value space, and we may simplify it as $\theta^{kl}$ for ease of notations. 
		\par
		Next, we will show the proposition is a direct consequence of the concentration inequality of Gaussian distribution. We prove our results under one-dimensional Fourier transformation under dimension $M$, where the proof can be easily generalized to a two-dimensional case. Owing to the definition of the norm of complex value, for any specific $u$,  
		\begin{equation}\label{}
			\small
			\begin{aligned}
				\left\|F_{\beps}(u)\right\|^{2} = F_{\beps}(u)\overline{F_{\beps}(u)} = \frac{1}{M^{2}}\beps^{\top}\Lambda\textbf{1}\textbf{1}^{\top}\bar{\Lambda}\bar{\beps}, 
			\end{aligned}
		\end{equation}
		where $\Lambda = \diag(e^{-i\theta^{0}}, \cdots, e^{-i\theta^{M - 1}})$. Then let $\bP = (\sqrt{1/M}\textbf{1}^{\top}, \cdots, )^{\top}\bar{\Lambda}$, where $(\sqrt{1/M}\textbf{1}^{\top}, \cdots, )^{\top}$ is constructed by vector $\sqrt{1/M}\textbf{1}$ and its orthogonal complement. We can verify that $\bP$ is an orthogonal matrix. Then, let $\beps = \bP^{\top}\by$, so that $\by$ has the same distribution with $\beps$. Thus  
		\begin{equation}
			\small
			\frac{1}{M^{2}}\beps^{\top}\Lambda\textbf{1}\textbf{1}^{\top}\bar{\Lambda}\beps = \frac{1}{M^{2}}\by^{\top}\bP\Lambda\textbf{1}\textbf{1}^{\top}\bar{\Lambda}\bar{\bP}^{\top}\bar{\by} = \frac{1}{M}\be_{1}^{\top}\by\bar{\by}^{\top}\be_{1} = \frac{1}{M}(\by^{1})^{2},
		\end{equation}
		where $\by^{1}$ is a standard Gaussian. Thus, by the Berstein's inequality to sub-exponential random variable i.e., $\chi_{1}^{2}$, we have 
		\begin{equation}
			\small
			\bbP\left(\left|F_{\beps}(u) - \mE\left[F_{\beps}(u)\right]\right| \geq \delta\right) = \bbP\left(\left|(\by^{1})^{2} - \mE\left[\left(\by^{1}\right)^{2}\right]\right| \geq \delta\right) \leq 2\exp\left(-\frac{1}{8}\min\left\{\delta^{2}, \delta\right\}\right). 
		\end{equation}
		Since $\delta\in(0, 1)$. Thus with probability at least $1 - \delta / M$, we have 
		\begin{equation}
			\small
			\frac{1}{M} - \frac{1}{M}\sqrt{8\log{\frac{2M}{\delta}}}\leq \left\|F_{\beps}(u)\right\|^{2} \leq \frac{1}{M} + \frac{1}{M}\sqrt{8\log{\frac{2M}{\delta}}},
		\end{equation}
		which proves that the target $\|F_{\beps_{t}}(u, v)\|^{2}$ is of order $\cO(\frac{1}{M})$, so that verify our conclusion.   
\end{proof}

\section{Text-Image Alignment Metrics}\label{app:Text-Image Alignment Metrics}
In this paper, we mainly use three metrics as in \citep{huang2023t2i} to measure the alignment between text prompt condition and generated image. Next, we give a brief introduction to the these metrics. 

\paragraph{CLIPScore (Text).} After extracting the features of generated images and text prompt respectively by CLIP encoder \citep{radford2021learning}, CLIPScore (Text) is the cosine similarity between the two features. Similarly, the CLIPScore (Image) is the cosine-similarity between two image features. 

\paragraph{BLIP-VQA.} To improve the limited details capturing capability of the CLIP encoder, \citep{huang2023t2i} propose BLIP-VQA which leverages the visual question answering (VQA) ability of BLIP model \citep{li2022blip}. They compare the generated with the target text prompt separately described by several questions. For example, the prompt  ``A blue bird'' can be separated into questions ``a bird ?'', ``a blue bird ?'' etc. Then BLIP-VQA outputs the probability of ``Yes'' when comparing generated images and these questions.   

\paragraph{MiniGPT4-CoT.} The MiniGPT4-CoT \citep{huang2023t2i} combines a strong multi-modality question answering model MiniGPT-4 \citep{zhu2023minigpt} and Chain-of-Thought \citep{wei2022chain}. The metric is computed by feeding the generated images to MiniGPT-4, then sequentially asking the model two questions ``describe the image'' and ``predict the image-text alignment score''. By constructing such CoT, the multi-modal model will not ignore the details in generated images.  

\section{Ablation Study on Number of [\texttt{EOS}]}\label{app:Ablation Study on Nubmer of eos}

\begin{figure}[htbp]
		\centering
    	\includegraphics[scale=0.28]{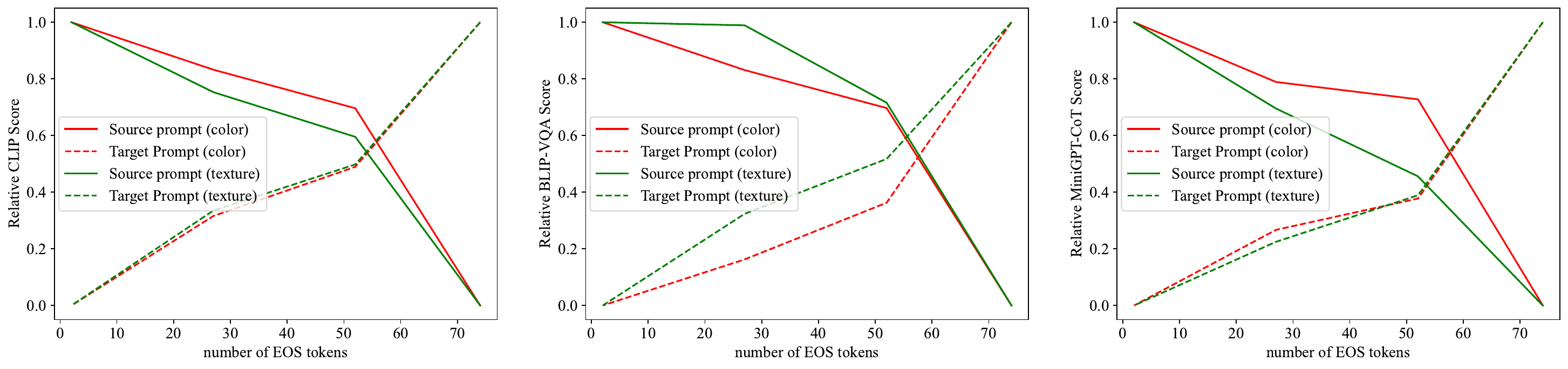}
		\caption{Relative BLIP-VQA Combine color and texture add complex, CLIP score, MiniGPT-CoT with source (or target) prompt under different number of [\texttt{EOS}]. Here the y-axis is the current BLIP-VQA, CLIP, MiniGPT-CoT over the maximum ones, which is used to alleviate the bias brought by the metric itself.}
%		 \vspace{-0.1in}
		\label{fig:num_eos}
	\end{figure}
In this section, we conduct an ablation on the number of [\texttt{EOS}]. Concretely, for each text prompt in \texttt{S-PromptSet}, we repeat the semantic tokens in the text prompt e.g., ``[\texttt{SOS}] + a yellow cat a yellow ... + [\texttt{EOS}]'', so that the number of [\texttt{EOS}] is reduced in these constructed text prompts. Then, we generate images under these reconstructed text prompts and compare the text-image alignments between generated images and source or target prompts. 
\par
The results are in Figure \ref{fig:num_eos}. As can be seen, with the increasing of semantic tokens (so that decreasing of [\texttt{EOS}]), the generated images tend to be consistent with the source prompt, instead of the target prompt. Therefore, we speculate that the domination of [\texttt{EOS}] may be partially originated from its larger number, compared with semantic tokens. On the other hand, we observe that [\texttt{EOS}] in the forward positions have larger impacts compared to the latter ones, as the alignments between generated images with target prompts significantly decreased along the x-axis from right to left, in Figure \ref{fig:num_eos}. This trends further indicate that the [\texttt{EOS}] may contain more information compared with the latter ones, which indicates the domination of [\texttt{EOS}] originates from their larger number, but also the more information in the first few [\texttt{EOS}].   

\section{Conveyed Information in Semantic Tokens}\label{sec:Conveyed Information in Semantic Tokens}
%\begin{table}[t]
%	%		\vspace{-0.1in}
%	\caption{The appearance of objects and attributes from source prompt in generated images under text prompts from \texttt{S-PromptSet}.}
%	%		\vspace{-0.1in}
%	\label{tbl:eos object or attribute}
%	% \vspace{0.1in}
%	\centering
%	\scalebox{0.8}{
%		{
%			\begin{tabular}{l|cc|cc}
%				\hline
%				\multirow{2}{*}{\diagbox[height=2\line]{Alignment}{Semantic}{Attribute}} & \multicolumn{2}{c|}{Color}  & \multicolumn{2}{c}{Texture} \\
%				&  Object &  Attribute          & Object & Attribute             \\
%				\hline
%				%					CLIPScore (Text)$\uparrow$ &  0.3231  &  \textbf{0.6769}
%				%					&  0.2316  &  \textbf{0.7684} \\
%				BLIP-VQA$\uparrow$  &     0.4665      &      \textbf{0.7870}    &    0.4411     &  \textbf{0.5796}   \\
%				%					MiniGPT-CoT$\uparrow$  &           &          &         &     \\
%				\hline
%			\end{tabular}
%			%				\vspace{-0.3in}
%	}}
%\end{table}

During our discussion in Section \ref{sec:The Working Mechanism of Text Prompt}, we conclude that the \texttt{[EOS]} has a larger impact than the ones of semantic tokens during T2I generation. However, as observed in Figure \ref{fig:eos substitbution generated image}, under text prompt with switched [\texttt{EOS}], some information in semantic tokens are still conveyed in the generated images, e.g., blue color in the last image of the first row in Figure \ref{fig:eos substitbution generated image}. Therefore, we explore how this information is conveyed in this section, which also reveals the working mechanism of text prompt.  
\par
Firstly, in Figure \ref{fig:cross-attention map semantic tokens}, we visualize the cross-attention map of each tokens under text prompt from \texttt{S-PromptSet}, similar to Figure \ref{fig:cross-attention map ratio}. Surprisingly, we find that in cross-attention map of semantic tokens and [\texttt{EOS}] are all visually similar to the shape of the final generated images. The similarity is reasonable for [\texttt{EOS}] as it contains the overall information, so that it is perceptible and transfer their information according to constructed cross-attention map. 
\par
On the other hand, when semantic tokens convey their information according to the similar cross-attention, for attributes (color or texture), unlike object information, they are potentially not contradict to overall shape decided by [\texttt{EOS}]. Thus, the information of attributes is more likely to be conveyed in its corresponding pixels. However, this does hold for object/noun tokens whose information is very likely related to shape, which has already been decided by [\texttt{EOS}]. 
\par
This discussion explains the phenomenon of information in semantic tokens are appeared in the generated images under prompt from \texttt{S-PromptSet}. Combining the observations to the working stage of text prompt in Section \ref{sec:The Working Mechanism of Text Prompt}, we can conclude that the semantic tokens also work in the T2I generation, though it has less impact compared to \texttt{[EOS]}. 
%\par

\begin{figure}[t!]
	\centering
	\includegraphics[scale=0.5]{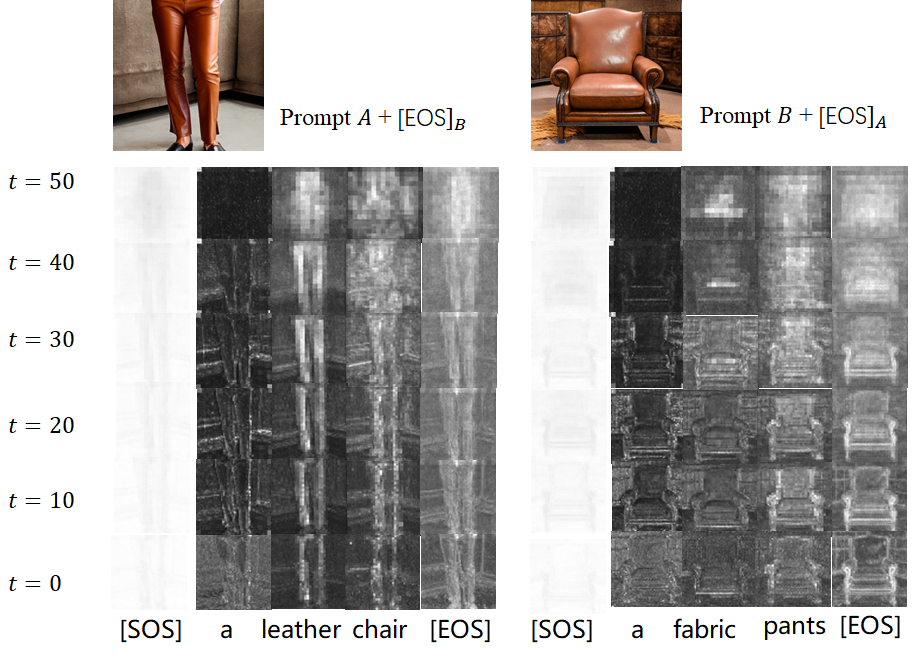}
	\caption{The visualization of cross-attention map under text prompt with switched [\texttt{EOS}] from \texttt{S-PromptSet}. The pixels corresponding to semantic tokens are in the shape of the final generated data as in text prompt $B$ provide [\texttt{EOS}]. For example, the token ``chair'' corresponds to pixels in the shape of paint, so its information can not be conveyed, while this phenomenon does exist in the attribute token ``leather''.}
	% \vspace{-0.2in}
	\label{fig:cross-attention map semantic tokens}
\end{figure}

% To quantitatively verify our discussion, we generate images under \texttt{S-PromptSet} and measure the strength of conveyed semantic information. Concretely, we compare the appearances of attribute and object from source prompt in generated images, where the appearance is measured under BLIP-VQA, as we find CLIPScore and MiniGPT-4 are inaccurate in measuring image's partial information (especially for attributes). The results are in Table \ref{tbl:eos object or attribute}, as we expected, which show the appearance chance of attributes in source prompt are significantly larger the ones of objects.     
\par
    \begin{figure}[t]
    	\centering
    	\includegraphics[scale=0.55]{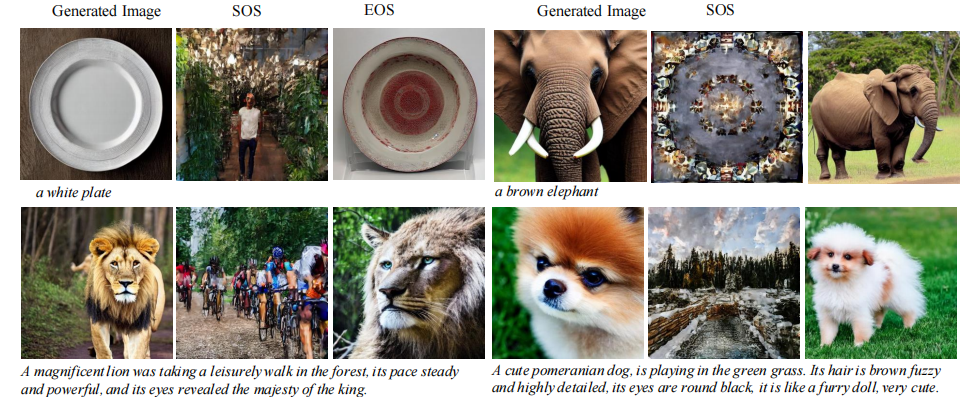}
    	%    	\vspace{-0.2in}
    	\caption{Generated images with prompts only contain information from [SOS] or [EOS].}
    	%    	\vspace{-0.2in}
    	\label{fig:sos substitution}
    \end{figure}
    
    \begin{figure}[t]
    	\centering
    	\includegraphics[scale=0.6]{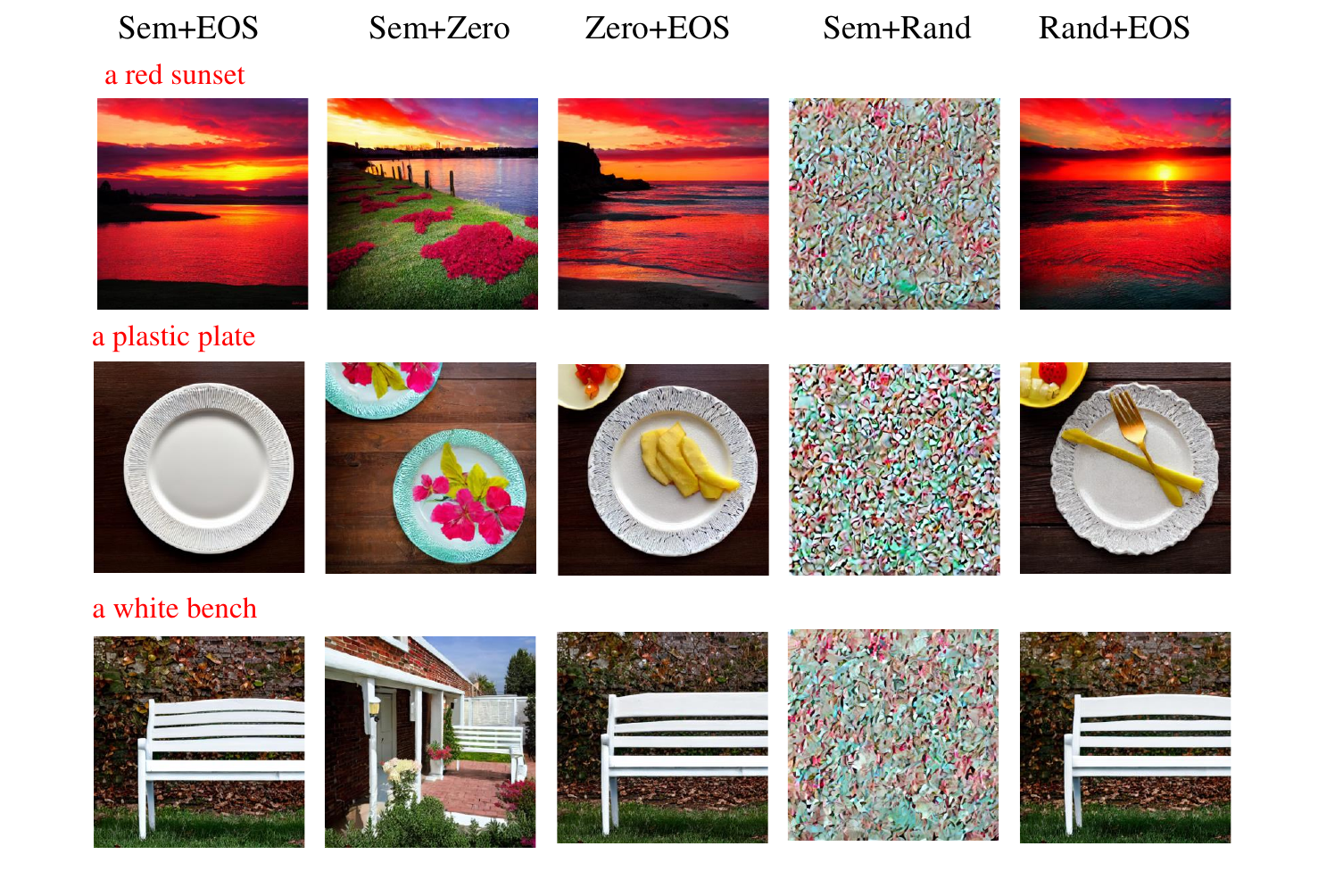}
%    	\vspace{-0.2in}
    	\caption{Generated images with zero or random vectors substitution.}
%    	\vspace{-0.2in}
    	\label{fig:zero_rand_substitution}
    \end{figure}
    
\section{[SOS] Contains no Textual Information}
As mentioned in Section \ref{sec:eos contains more information}, the special token [SOS] is supposed to contain no textual information, due to the auto-regressive textual prompt encoder. To further verify this, we conduct the following two types of prompts. 1): all 77 tokens are [SOS] from the given text prompt, 2) except the first [SOS] token, all other 76 
tokens are [EOS] from the given text prompt. Then, we generate images under these prompts with SD v1.5. The generated images as in Figure 
\ref{fig:sos substitution}. As can be seen, for the first type of prompt, no textual information is conveyed, while the phenomenon disappears in the second type of prompt. This observation verifies our conclusion that [SOS] contains no textual information.

\section{More Evidences on [\texttt{EOS}] Contains More Information}
	\begin{table}[t]
	%		\vspace{-0.1in}
	\caption{The alignment of generated results image under different constructed text prompt sets. Here ``Sem  + EOS'' is the original text prompt, and serves as baseline here. Besides that, the CLIPScore (Image) is the image-level alignment of generated images with the ones under ``Sem  + EOS''.}
	%		\vspace{-0.1in}
	\label{tbl:substitution}
	% \vspace{0.1in}
	\centering
	\scalebox{1.0}{
		{
			\begin{tabular}{l|cccc}
				\hline
				Text Prompt & CLIPScore (Text)$\uparrow$  & CLIPScore (Image)$\uparrow$ & BLIP-VQA$\uparrow$ & MiniGPT4-CoT$\uparrow$ \\
				\hline 
				Sem + Zero  &  0.2407   &  0.6732 &  0.5392 & 0.6757 \\
				Sem + Rand  &  0.2153   &  0.6038 &  0.3606 & 0.5428 \\
				Zero + EOS  &  0.2999   &  0.8887 &  0.7467 & 0.6982 \\
				Rand + EOS  &  0.3008   &  0.8791 &  0.7669 & 0.7000 \\
				Sem  + EOS  &  0.3110   &  1.0000 &  0.8999 & 0.7412 \\
				\hline
			\end{tabular}
%			\vspace{-0.3in}
	}}
\end{table}
In this section, we further verify that the impact of [\texttt{EOS}] is larger than the ones of semantic tokens in T2I generation. To further verify this conclusion, under the text prompts with the format of ``[\texttt{SOS}] + Sem + [\texttt{EOS}]'' from our dataset \texttt{PromptSet}, we substitute all semantic tokens or [\texttt{EOS}] with zero vectors or random Gaussian noise. As a result, we get the 4 sets of text prompts, i.e., ``[\texttt{SOS}] + Sem + Zero'' (abbrev Sem + Zero), ``[\texttt{SOS}] + Sem + Random'' (Sem + Rand), ``[\texttt{SOS}] + Zero + [\texttt{EOS}]'' (Zero + EOS), and ``[\texttt{SOS}] + Random + [\texttt{EOS}]'' (Rand + EOS). These constructed text prompts ideally contain complete semantic information, and we verify the alignment of the generated images with the corresponding text prompt conditions. The alignments are measured by text-image alignment metrics: CLIPScore \citep{radford2021learning,hessel2021clipscore}, BLIP-VQA \citep{li2022blip,huang2023t2i}, and  MiniGPT4-CoT \citep{zhu2023minigpt,huang2023t2i}. 
\par
The results are summarized in Table \ref{tbl:substitution}. As can be seen, as expected for baseline combination ``Sem + EOS'', the alignments under text prompts with ``EOS'' preserved are significantly better than the ones with ``Sem'' preserved. Thus, the observations further verify our conclusion that  \emph{the [\texttt{EOS}] has larger influences than semantic tokens during the denoising process.} 
\par
Moreover, we find the generation is somehow robust, as involving random noise in text prompts still generates semantic meaningful images. We visualize some generated images under constructed text-prompts from Table \ref{tbl:substitution} in Figure \ref{fig:zero_rand_substitution}, which indicates the images under ``Zero + EOS'' indeed visually have the best quality in alignment, so that consist with Table \ref{tbl:substitution}. Besides that, the other combinations generate semantic meaningful images as well, expected for ``Sem + Rand'' Thus semantic tokens do not contain enough information for generation.
\section{Key or Value Dominates the Influence?}
     \begin{figure}[htbp]
	\centering
	\includegraphics[scale=0.4]{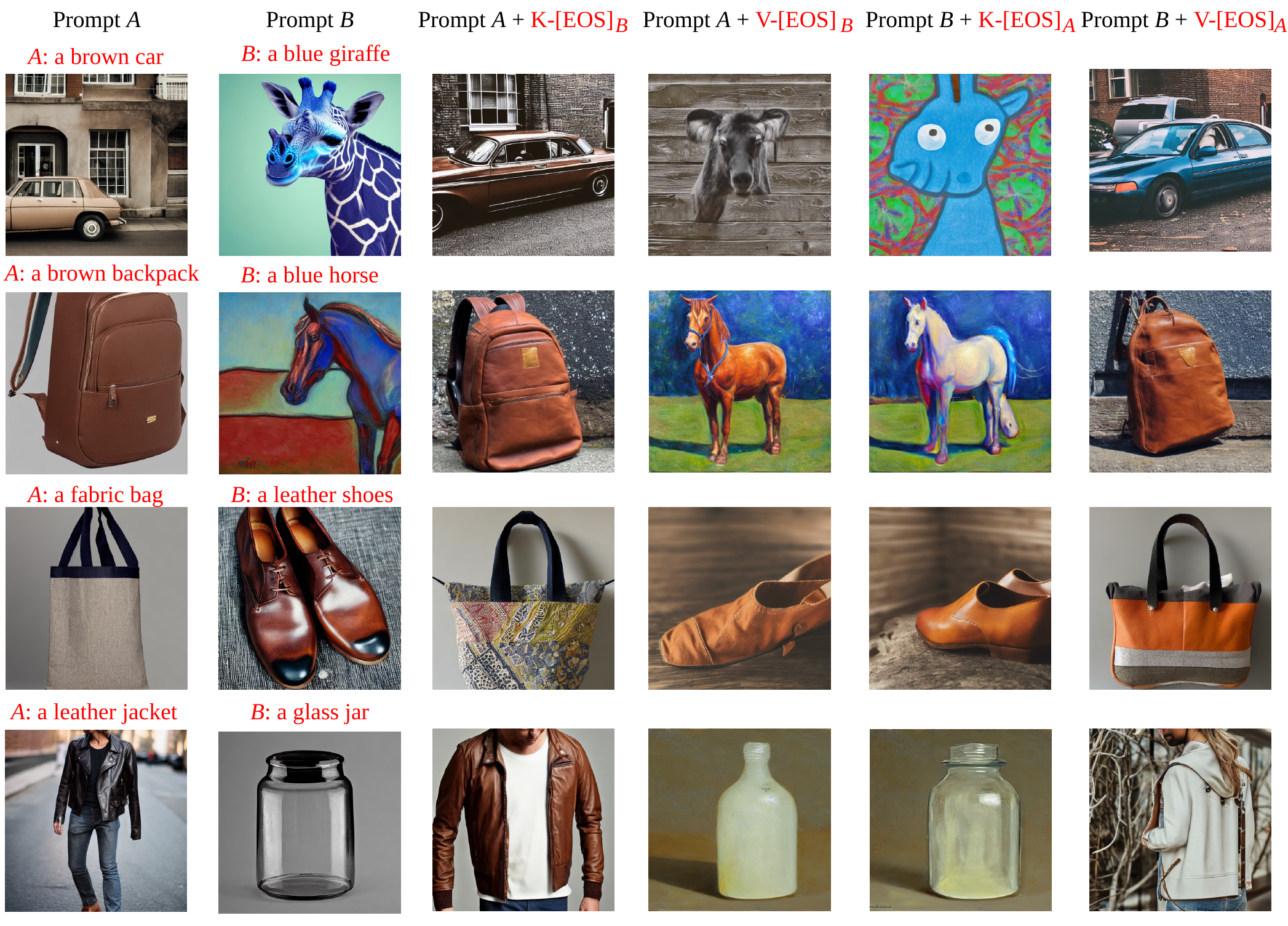}
	\vspace{-0.1in}
	\caption{Generated examples of Key or Value substitution.}
	\label{fig:kv substitution}
\end{figure}

\begin{figure*}[t!]
	\centering
	\includegraphics[scale=0.3]{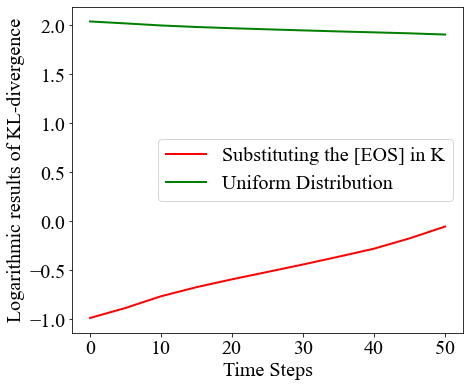}
	%        \vspace{-0.2in}
	\caption{The averaged KL-divergence over pixels and layers.}
	%        \vspace{-0.1in}
	\label{fig:divergence}
\end{figure*}
\begin{table*}[t]
%			\vspace{-0.1in}
	\caption{The alignment of generated image with its source and target prompts, under switched [\texttt{EOS}] on Key or Value substitution. Here $KV$-Sub is the complete substitution as in Section \ref{sec:The Working Mechanism of Text Prompt}, which serves as a baseline here.}
	\label{tbl:eos kv substitution alignment}
	% \vspace{0.1in}
	\centering
	\scalebox{1.0}{
		{
			\begin{tabular}{l|cccccc}
				\hline
				\multirow{2}{*}{\diagbox{Alignment}{Attribute}} & \multicolumn{2}{c}{$K$-Sub}  & \multicolumn{2}{c}{$V$-Sub} & \multicolumn{2}{c}{$KV$-Sub}  \\
				&  Source &  Target          & Source & Target  & Source &  Target                     \\
				\hline
				
				CLIPScore (Text)$\uparrow$ &\textbf{0.3132}    &0.1875  &0.2538    &\textbf{0.2628}  &0.2363  &\textbf{0.2758} 
				  \\
				BLIP-VQA$\uparrow$ &\textbf{0.7127}    &0.2379  &0.3724    &\textbf{0.3984}  &0.3325  &\textbf{0.4441}     \\
				MiniGPT-CoT$\uparrow$  &\textbf{0.8025}    &0.6215  &0.6749    &\textbf{0.7071}  &0.6473  &\textbf{0.7213}           \\
				\hline
			\end{tabular}
%							\vspace{-0.4in}
	}}
\end{table*}

\begin{figure}[t!]
	\centering
	\includegraphics[scale=0.4]{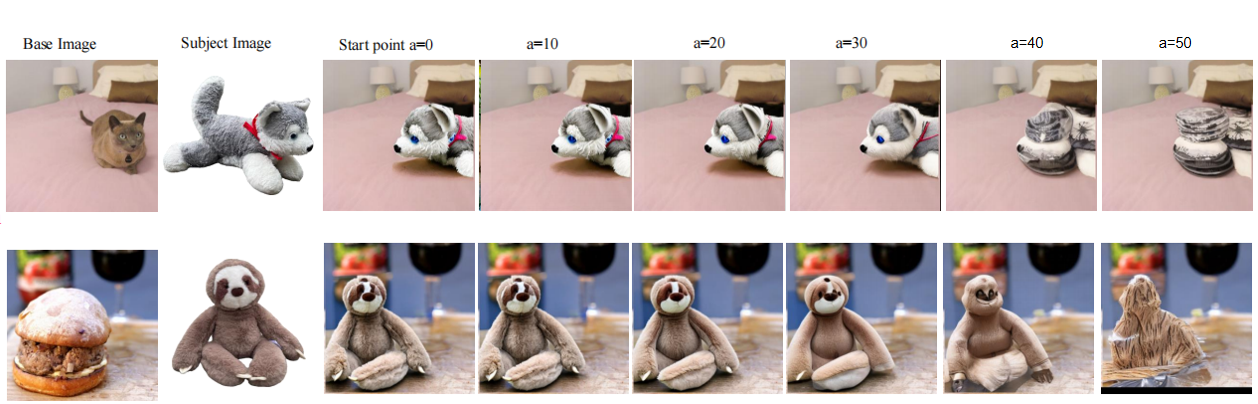}
	%            \vspace{-0.1in}
	\caption{We implement our sampling strategy on subject-driven generation model, AnyDoor \citep{chen2024anydoor}. We remove the condition image from different time steps (denote as $a$) during denoising process. 
		The generated images still preserve the specific details as baseline model (start point $a=0$) when start removing time steps is set to 20.}
	%            \vspace{-0.2in}
	\label{fig:subject driven}
\end{figure}

\begin{figure}[t!]
	\centering
	\includegraphics[scale=0.45]{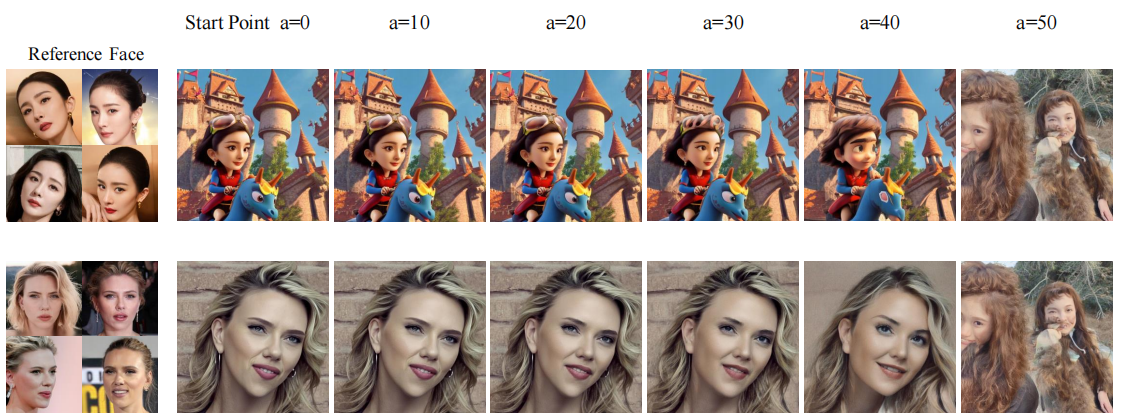}
	%            \vspace{-0.1in}
	\caption{We implement our sampling strategy on human face generation model, PhotoMaker. We remove the condition (text prompts and reference face) from different time 
		steps (denote as $a$) during denoising process. The generated images and faces still preserve the specific details as baseline model (start point $a=0$) when start removing time steps is set to 20.}
	%            \vspace{-0.2in}
	\label{fig:face generation}
\end{figure}

	\begin{figure}[t!]
		\centering
		\includegraphics[scale=0.6]{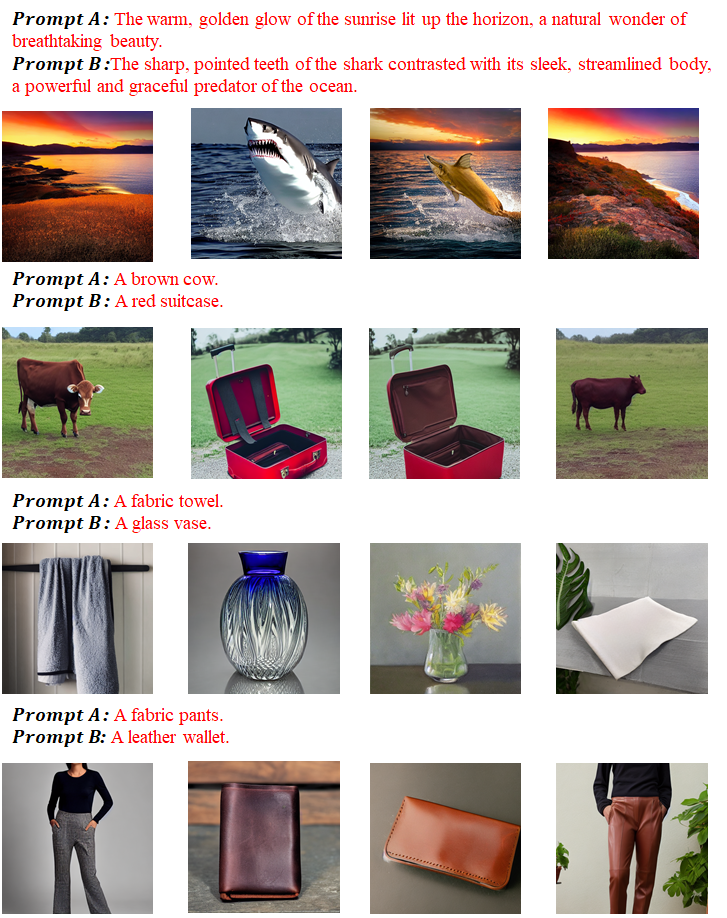}
%            \vspace{-0.1in}
		\caption{More generated examples under tokens from \texttt{S-PromptSet}.}
%            \vspace{-0.2in}
		\label{fig:eos substitution}
	\end{figure}
	
%	\begin{figure}[t!]
%		\centering
%		\includegraphics[scale=0.5]{./pic/eos/eos_substitution_step_appendix.pdf}
%		\caption{Generated examples when with switched [\texttt{EOS}] during denoising steps $[a, 50]$.}
%		\label{fig:eos substitution step}
%	\end{figure}
%	
%	\begin{figure*}[t!]
%		\centering
%		\includegraphics[scale=0.43]{./pic/eos/semantic_token_convey_vis.pdf}
%		\caption{text prompt ``[\texttt{SOS}] + Prompt $A$ + [\texttt{EOS}]$_{B}$'' in steps $[0, a]$ and ``[\texttt{SOS}] + Sem + [\texttt{EOS}]$_{B}$'' in steps $[a, 50]$.}
%		 \vspace{-0.2in}
%		\label{fig:semantic_convey_vis}
%	\end{figure*}
	
	As mentioned in Section \ref{sec:preliminaries}, the information from [\texttt{EOS}] is conveyed by the cross-attention module. More concretely, the Key ($K$) and Value ($V$) in it, respectively decide the weights and features in the output of the cross-attention module (a weighted sum of features). Next, we explore their individual influence for [\texttt{EOS}] to further reveal the working mechanism of it.          
	\par
	Concretely, as in Section \ref{sec:The Working Mechanism of Text Prompt}, we generate images under the constructed text prompt set \texttt{S-PromptSet}. However, the substitution of [\texttt{EOS}] is only conducted on computing Key or Values in cross-attention module, which are respectively denoted as $K$-Substitution (Sub) and $V$-Sub. For such two substitutions, similar to Table \ref{tbl:eos substitution alignment}, we compare the image-text alignment of generated images with source and target prompts. 
	\par
	The results are summarized in Table \ref{tbl:eos kv substitution alignment}. As can be seen, substituting the [\texttt{EOS}] in $V$ has a larger influence than the substitutions in $K$. To explain this, as we have observed in Figure \ref{fig:token_cross_atten}, the weights on semantic tokens and [\texttt{EOS}] are significantly smaller than the ones on [\texttt{SOS}]. However, the $K$ of [\texttt{EOS}] is only related to these small weights, which have limited influence. In contrast to $K$, the $V$ of [\texttt{EOS}] contains information on features, which can be directly conveyed in generated images. So that we can conclude the value of [\texttt{EOS}] dominates the influence in generation as observed in Table \ref{tbl:eos kv substitution alignment}. We present some generated images under text prompts from \texttt{S-PromptSet}, but with only Key or Value substituted as in Table \ref{tbl:eos kv substitution alignment}. The generated images are in Figure \ref{fig:kv substitution}. 
    \par
    We further verify the variation of the cross-attention map after substituting [\texttt{EOS}]. For each pixel, the cross-attention map of it is a discrete probability distribution. Thus, we compute the KL-divergence \citep{vapnik1999nature} between probability distributions under substituted/unsubstituted $K$. The averaged KL-divergence over pixels and layers in models under different denoising steps is presented in Figure \ref{fig:divergence}, where we add the KL-divergence of cross-attention map distribution between a uniform distribution as a baseline. The result shows that even with substituted [\texttt{EOS}], the cross-attention map does not vary much. We speculate that this is because as in Figure \ref{fig:token_cross_atten}, the weights in [\texttt{SOS}] dominate the cross-attention map. Thus, in the cross-attention module, altering $K$ has a slighter influence compared with altering $V$.

\section{Firstly Conveyed Information for Subject-Driven Generation}\label{app:subject}
As mentioned in Section \ref{sec:eos contains more information}, the injected textual information is conveyed in the first stage of diffusion model. Since the information is conveyed in the cross-attention module of model, we speculate this phenomenon may be generalized to the other conditional generation task, e.g., subject-driven generation with an extra image as condition. 
\par
To see this, we conduct experiments under two tasks: zero-shot subject-driven generation and human face generation. For such two tasks, there is extra reference image (given subject and human face) used as condition to guide image generation. We use the sampling strategy as in Figure \ref{fig:substitution} for the two tasks, respectively follow the backbone methods AnyDoor \citep{chen2024anydoor} and Photomaker \citep{li2024Photomaker}. The generated images are in Figure \ref{fig:subject driven} and \ref{fig:face generation}, respectively. 
\par
As can be seen, the generation results verify the conclusion that for conditional information from in the other modality than text, they will be also firstly conveyed in during the diffusion process.

\section{More Generated Images}\label{eq:app:more generated images}  
%\subsection{Zero or Random Substitution}\label{app:Zero or Random Substitution}
%In this subsection, we compare the generated images under constructed text prompt's combinations in Table \ref{tbl:substitution}, e.g., ``Sem + Zero''. The generated images are presented in Figure \ref{fig:zero_rand_substitution}. As can be seen, preserving [\texttt{EOS}] generates images has better consistency with target prompt, compared to the ones with Sem preserved. 

\subsection{[\texttt{EOS}] Substitution}\label{app:eos Substitution}
In this subsection, we first present the generated images under text prompts from \texttt{S-PromptSet} in Figure \ref{fig:eos substitution}. As can be seen, most overall shape of generated images are consistent with the ones conveyed by \texttt{[EOS]}.  
%\par
%Next, we present the generated images under text prompts with \texttt{[EOS]} substituted in different denoising steps range, i.e., in $[a, 50]$ steps. The images are in Figure \ref{fig:eos substitution step}. As can be seen, the information of [\texttt{EOS}] are already conveyed in the first several denosing steps of generation, because the generated images are similar to target prompt when $a$ is close to 50 e.g., $a = 20$. This verifies our conclusion in Section \ref{sec:Decides the Overall Information}.  

\subsection{Generated Images in Paragraph ``Acceleration of Sampling'' of Section \ref{sec:application}}
Next, we present more generated images under noise prediction \eqref{eq:new noise prediction} with varied $a$ in Figure \ref{fig:more substitution ddim} and \ref{fig:more substitution dpms}. 

\begin{figure}[h!]
	\centering
	\includegraphics[scale=0.6]{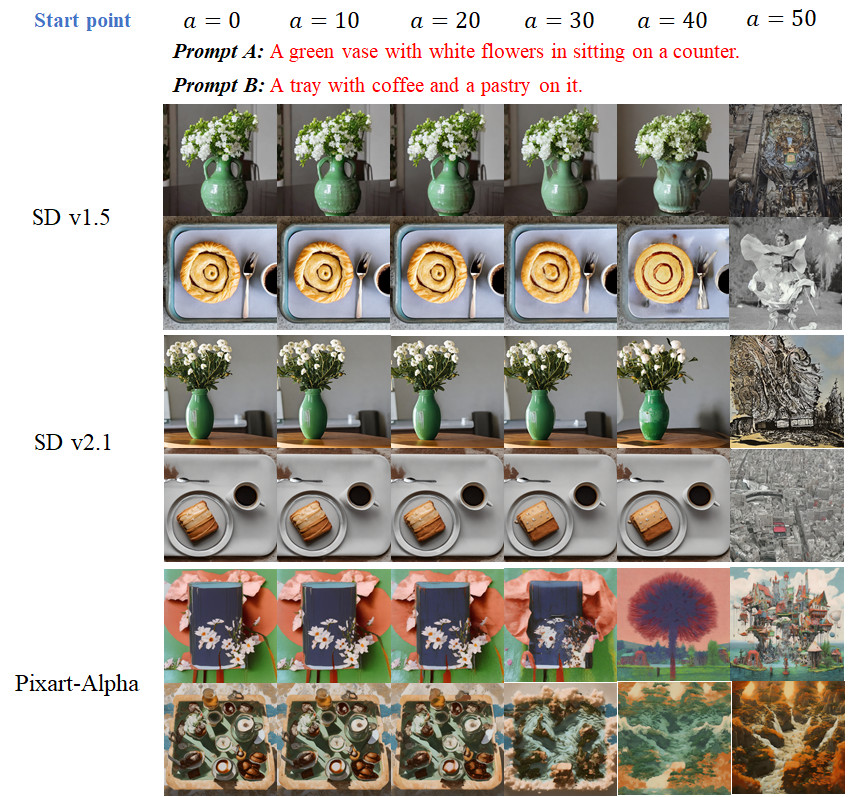}
	%            \vspace{-0.1in}
	\caption{The generated images with 50 steps DDIM under $\beps_{\btheta}$ in \eqref{eq:new noise prediction}, where the textual information are $\cC$ removed during time steps $t\in[0, a]$. With $a\to 50$, the inference cost is decreased.}
	%            \vspace{-0.2in}
	\label{fig:more substitution ddim}
\end{figure}

\begin{figure}[h!]
	\centering
	\includegraphics[scale=0.65]{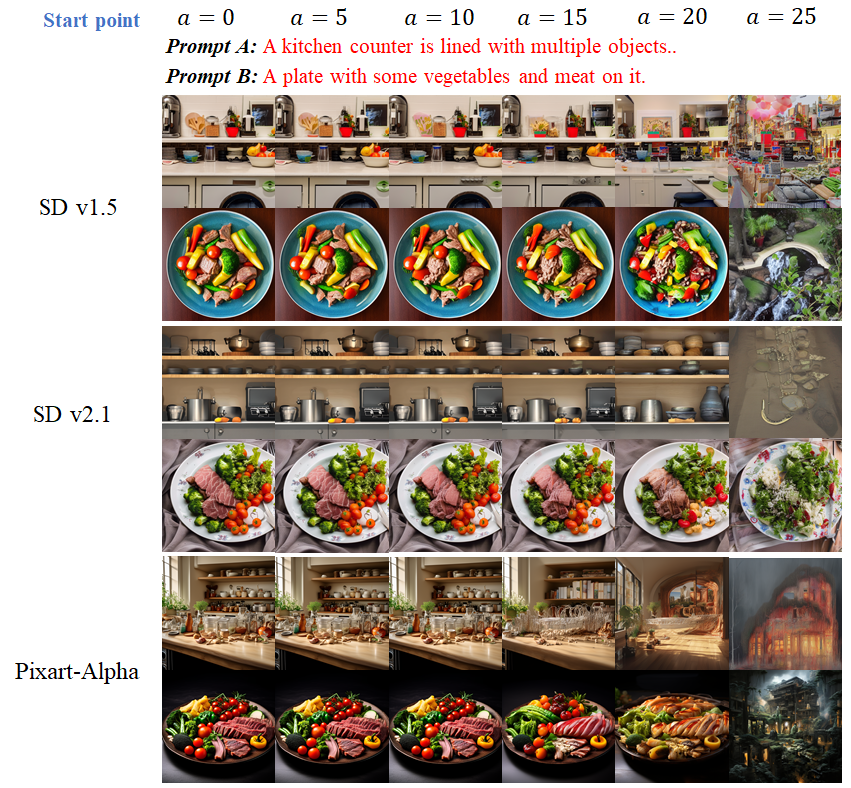}
	%            \vspace{-0.1in}
	\caption{More generated images with 25 steps DPM-Solver under $\beps_{\btheta}$ in \eqref{eq:new noise prediction}, where the textual information are $\cC$ removed during time steps $t\in[0, a]$. With $a\to 25$, the inference cost is decreased.}
	%            \vspace{-0.2in}
	\label{fig:more substitution dpms}
\end{figure}

\newpage 

\section*{NeurIPS Paper Checklist}
\begin{enumerate}
	
	\item {\bf Claims}
	\item[] Question: Do the main claims made in the abstract and introduction accurately reflect the paper's contributions and scope?
	\item[] Answer: \answerYes{} % Replace by \answerYes{}, \answerNo{}, or \answerNA{}.
	\item[] Justification: The main contributions of this paper have been clarified in Abstract. 
	
	\item {\bf Limitations}
	\item[] Question: Does the paper discuss the limitations of the work performed by the authors?
	\item[] Answer: \answerYes{} % Replace by \answerYes{}, \answerNo{}, or \answerNA{}.
	\item[] Justification: The limitation of this paper is discussed in the main paper.
	
	\item {\bf Theory Assumptions and Proofs}
	\item[] Question: For each theoretical result, does the paper provide the full set of assumptions and a complete (and correct) proof?
	\item[] Answer: \answerNA{} % Replace by \answerYes{}, \answerNo{}, or \answerNA{}.
	\item[] Justification: 
	
	\item {\bf Experimental Result Reproducibility}
	\item[] Question: Does the paper fully disclose all the information needed to reproduce the main experimental results of the paper to the extent that it affects the main claims and/or conclusions of the paper (regardless of whether the code and data are provided or not)?
	\item[] Answer: \answerYes{} % Replace by \answerYes{}, \answerNo{}, or \answerNA{}.
	\item[] Justification: The experimental details are in main part of this paper. 

	\item {\bf Open access to data and code}
	\item[] Question: Does the paper provide open access to the data and code, with sufficient instructions to faithfully reproduce the main experimental results, as described in supplemental material?
	\item[] Answer: \answerNo{} % Replace by \answerYes{}, \answerNo{}, or \answerNA{}.
	\item[] Justification: We will release the code in future. 

	\item {\bf Experimental Setting/Details}
	\item[] Question: Does the paper specify all the training and test details (e.g., data splits, hyperparameters, how they were chosen, type of optimizer, etc.) necessary to understand the results?
	\item[] Answer: \answerYes{} % Replace by \answerYes{}, \answerNo{}, or \answerNA{}.
	\item[] Justification: They are specified in main part of this paper. 
	
	\item {\bf Experiment Statistical Significance}
	\item[] Question: Does the paper report error bars suitably and correctly defined or other appropriate information about the statistical significance of the experiments?
	\item[] Answer: \answerNA{} % Replace by \answerYes{}, \answerNo{}, or \answerNA{}.
	\item[] Justification: The results have no error bars.  
	
	\item {\bf Experiments Compute Resources}
	\item[] Question: For each experiment, does the paper provide sufficient information on the computer resources (type of compute workers, memory, time of execution) needed to reproduce the experiments?
	\item[] Answer: \answerYes{} % Replace by \answerYes{}, \answerNo{}, or \answerNA{}.
	\item[] Justification: 
	
	\item {\bf Code Of Ethics}
	\item[] Question: Does the research conducted in the paper conform, in every respect, with the NeurIPS Code of Ethics \url{https://neurips.cc/public/EthicsGuidelines}?
	\item[] Answer: \answerYes{} % Replace by \answerYes{}, \answerNo{}, or \answerNA{}.
	\item[] Justification:

	\item {\bf Broader Impacts}
	\item[] Question: Does the paper discuss both potential positive societal impacts and negative societal impacts of the work performed?
	\item[] Answer: \answerNA{} % Replace by \answerYes{}, \answerNo{}, or \answerNA{}.
	\item[] Justification:
	
	\item {\bf Safeguards}
	\item[] Question: Does the paper describe safeguards that have been put in place for responsible release of data or models that have a high risk for misuse (e.g., pretrained language models, image generators, or scraped datasets)?
	\item[] Answer: \answerNA{} % Replace by \answerYes{}, \answerNo{}, or \answerNA{}.
	\item[] Justification: 
	
	\item {\bf Licenses for existing assets}
	\item[] Question: Are the creators or original owners of assets (e.g., code, data, models), used in the paper, properly credited and are the license and terms of use explicitly mentioned and properly respected?
	\item[] Answer: \answerNA{} % Replace by \answerYes{}, \answerNo{}, or \answerNA{}.
	\item[] Justification: 
	
	\item {\bf New Assets}
	\item[] Question: Are new assets introduced in the paper well documented and is the documentation provided alongside the assets?
	\item[] Answer: \answerNA{} % Replace by \answerYes{}, \answerNo{}, or \answerNA{}.
	\item[] Justification: 
	
	\item {\bf Crowdsourcing and Research with Human Subjects}
	\item[] Question: For crowdsourcing experiments and research with human subjects, does the paper include the full text of instructions given to participants and screenshots, if applicable, as well as details about compensation (if any)? 
	\item[] Answer: \answerNA{} % Replace by \answerYes{}, \answerNo{}, or \answerNA{}.
	\item[] Justification:
	
	\item {\bf Institutional Review Board (IRB) Approvals or Equivalent for Research with Human Subjects}
	\item[] Question: Does the paper describe potential risks incurred by study participants, whether such risks were disclosed to the subjects, and whether Institutional Review Board (IRB) approvals (or an equivalent approval/review based on the requirements of your country or institution) were obtained?
	\item[] Answer: \answerNA{} % Replace by \answerYes{}, \answerNo{}, or \answerNA{}.
	\item[] Justification: 
	
\end{enumerate}